\documentclass[11pt,oneside,letter]{article}
\usepackage[top=1 in, bottom=1 in, left=0.85 in, right=0.85 in]{geometry}

\usepackage{amsmath,amsfonts,bm}

\def\eqref#1{equation~\ref{#1}}

\def\1{\bm{1}}

\DeclareMathAlphabet{\mathsfit}{\encodingdefault}{\sfdefault}{m}{sl}
\SetMathAlphabet{\mathsfit}{bold}{\encodingdefault}{\sfdefault}{bx}{n}

\newcommand{\E}{\mathbb{E}}

\DeclareMathOperator*{\argmax}{arg\,max}

\usepackage{float}
\usepackage{times}
\usepackage{mathtools}
\usepackage[round]{natbib}
\usepackage{makecell}
\usepackage{amsmath,amssymb,graphicx,url}
\usepackage{thmtools,thm-restate,wrapfig,enumitem,mathabx}
\usepackage{booktabs}
\usepackage{color}
\usepackage{colortbl}
\usepackage{multirow}

\usepackage[utf8]{inputenc} 
\usepackage[T1]{fontenc}    
\usepackage{hyperref}       
\usepackage{url}            
\usepackage{booktabs}       
\usepackage{amsfonts}       
\usepackage{nicefrac}       
\usepackage{microtype}      
\usepackage{xcolor}         

\usepackage{algorithm}
\usepackage[noend]{algpseudocode}
\usepackage{titletoc}
\usepackage{bbm}
\usepackage{caption} 
\usepackage{subcaption}

\newtheorem{theorem}{Theorem}

\newtheorem{lemma}{Lemma}

\newtheorem{proof}{Proof}

\newtheorem{proposition}{Proposition}

\newcommand{\prob}{\mathbb{P}}

\title{SP3O: Reinforcement Learning from Segment Preferences without Reward Modeling}
\date{}
\author{
	Evan Assmus \\
	University of Michigan, Ann Arbor \\  \texttt{eassmus@umich.edu} \\ 
	\and
	Qining Zhang \\
	University of Michigan, Ann Arbor \\  \texttt{qiningz@umich.edu} \\ 
	\and
	Lei Ying \\
	University of Michigan, Ann Arbor \\  \texttt{leiying@umich.edu}\\
	\\
}

\begin{document}

\maketitle

\begin{abstract}
Preference-based reinforcement learning (PbRL) for general stochastic MDPs often requires training a reward model. Existing reward-model-free methods are either restricted to bandits or deterministic MDPs, such as DPO or P3O, or use zeroth-order, gradient-free optimization, which in general exhibits a slower convergence rate than gradient-based algorithms. Furthermore, existing reward-model-free preference-based RL algorithms almost exclusively use trajectory-level feedback, which can require significant effort from a human evaluator when trajectories are long. On the other hand, segments are much shorter, so they are easier to compare and evaluate. In this paper, we introduce a novel reward-model-free, critic-free, and gradient-based PbRL algorithm compatible with segment preferences named Segment Pairwise Proximal Policy Optimization (SP3O). SP3O utilizes segment-level preference feedback to construct an accurate policy value difference estimator via off-policy importance sampling, and then uses the estimator to compute the policy gradient via a PPO-type loss function. We provide a theoretical basis for the algorithm and analyze the tradeoff in choosing the segment length. We also evaluate it experimentally against other PbRL/RLHF algorithms in robotic control and LLM finetuning settings to show its improved performance, especially in long-horizon tasks.
\end{abstract}

\section{Introduction}

Reinforcement Learning (RL) \citep{sutton1998reinforcement} has seen great success across the domains of game playing \citep{silver2016mastering, vinyals2019grandmaster, mnih2013playing} and robotic control \citep{schulman2017proximalpolicyoptimizationalgorithms, haarnoja2018soft}. However, applying RL to more complex real-world tasks remains challenging, where one major bottleneck is designing a reward function that will naturally encourage agents to learn the desired behavior and avoid reward hacking~\citep{skalse2022defining}. Preference-based RL (PbRL) has emerged as a powerful framework for addressing these issues, in which, instead of receiving a numerical reward at each interaction step, the agent actively queries preferences over behaviors from oracles (humans or AI systems~\citep{lee2023rlaif}).
The candidates to be compared usually take the form of pairs of full trajectories or segments. 
The oracle then returns its preferences about the candidates, which are then used to improve and optimize the policy. This method has seen great success in the domains of fine-tuning language models \citep{ouyang2022traininglanguagemodelsfollow, bai2022training, rafailov2023direct, shao2024deepseekmath}, continuous control tasks \citep{christiano2017deep}, game playing \citep{christiano2017deep, ibarz2018reward}, and fine-tuning image generation models \citep{lee2023aligning}. 

\textbf{Reward Model.} Many existing preference-based RL algorithms, including Reinforcement Learning from Human Feedback (RLHF) for LLMs \citep{ouyang2022traininglanguagemodelsfollow, bai2022training, shao2024deepseekmath}, work by training an additional neural network, called the \emph{reward model}, and using a classic reward-based RL method, such as PPO \citep{schulman2017proximalpolicyoptimizationalgorithms}, to optimize the policy \citep{christiano2017deep}
from the numeric reward generated by the reward model. 
Depending on the application, the reward model can be trained offline, online, or a combination of both. The process generally assumes that preferences come from a Bradley-Terry model~\citep{bradleyterry} based on an unknown true reward function, and then maximizes the log-likelihood of seeing the preferences. Specifically, most of these works assume that preference processes are determined by either the partial reward sum of a trajectory segment \citep{christiano2017deep, ibarz2018reward} or the total reward of a trajectory \citep{ouyang2022traininglanguagemodelsfollow, bai2022training}, i.e., for two trajectories (segments) $\sigma^1$ and $\sigma^2$, both consisting of $L$ steps, the probability of preferring one candidate over the other is governed by the partial return model, i.e.,
\begin{align}\label{eq:BT}
    \prob\left[\sigma^1 \succ \sigma^2\right] = \mathtt{logistic}\left(  \sum_{k=1}^L r\left(s_k^1,a_k^1\right) - \sum_{k=1}^L r\left(s_k^2,a_k^2\right)\right),
\end{align}
where $\sigma^1=\{s_k^1, a_k^1\}_{k=1}^L$ and $\sigma^2=\{s_k^2, a_k^2\}_{k=1}^L$ are the states and actions in time sequence constituting trajectories (segments) $\sigma^1$ and $\sigma^2$ respectively, and $\mathtt{logistic}(\cdot)$ is the standard logistic function. One may also replace the logistic function with other link functions~\citep{zhang2025zeroth}.
This formulation has been successful in practice, but recent works raise questions about the adequacy of such a simple preference model, and variants have also been proposed~\citep{knox2022models, hejna2024contrastivepreferencelearninglearning}. More importantly, training a reward model is challenging: it complicates the training pipeline, requires additional memory and computational resources, and raises new concerns during evaluation, including distribution shift, lack of ground truth, and overfitting in joint training \citep{casper2023open}. Therefore, it is usually more appropriate to treat reward models as semi-quantitative evaluation tools instead of a perfect proxy of the true reward function.

\textbf{RL Directly from Preferences.} Methods that do not rely on reward model training, such as DPO~\citep{rafailov2023direct} and P3O~\citep{wu2023pairwiseproximalpolicyoptimization}, have also been introduced, but only under restrictive assumptions. For example, these algorithms were originally designed for contextual bandits and can only be generalized to MDPs with a deterministic transition, making them unsuitable for more complex real-world RL problems where stochasticity naturally arises from both environments and the interaction with humans. Furthermore, these algorithms require preferences to be between full trajectories that start from the same state.
For general RL settings, one approach is to use zeroth-order methods. For example, \citet{zhang2025zeroth,zhang2025rlhf} introduced a zeroth-order policy optimization algorithm that works by perturbing and comparing two policies to perform a zeroth-order policy update. Another approach is to use an evolutionary optimization strategy with preferences \citep{busa2014preference}. 
However, in optimization, it is known that zeroth-order and evolutionary approaches usually do not converge as fast as gradient-based first-order methods, especially when the size of the actor neural network scales up. 

\textbf{RLHF with Segments.} Evaluating short segments instead of long trajectories is often easier for human evaluators. Furthermore, when preferences are binary, evaluating more short segments gives more information than evaluating long segments even when the total length is the same. For these reasons algorithms have been developed that use preferences between segments \citep{wilson2012bayesian, christiano2017deep, lai2024step, guo2026segment}. However, these methods either require segments that start from the same state \citep{lai2024step, wilson2012bayesian, guo2026segment} which is not possible in some environments or use a reward model \citep{christiano2017deep, guo2026segment}.

In this paper, we consider a general stochastic MDP and aim to design a gradient-based algorithm for preference-based RL without reward modeling. Furthermore, for reasons discussed above, we are interested in RL with segment preferences, where segments do not necessarily start from the same state. In short, our goal is to train a policy directly from preferences between arbitrary segments.

\subsection{Main Contributions} In this paper, we introduce a novel PbRL algorithm using segment preferences without reward models, called \emph{\textbf{S}egment \textbf{P}airwise \textbf{P}roximal \textbf{P}olicy \textbf{O}ptimization} (SP3O). SP3O can be viewed as a generalization of P3O \citep{wu2023pairwiseproximalpolicyoptimization} to the stochastic MDP setting with segment preferences and allowing different starting states for preferences. 

\noindent{\bf Algorithm Design.} SP3O has the following characteristics:
\begin{itemize}
    \item SP3O is reward-model-free, critic-free, and applicable to general stochastic MDPs.
    \item SP3O uses segment-level preferences instead of trajectory-level preferences, which are more efficient to evaluate.  
    \item SP3O uses gradient-based policy optimization and a hybrid on-policy and off-policy approach to make efficient use of human preferences.
\end{itemize}

\noindent{\bf Theoretical Analysis.} We provide theoretical backup for the design of SP3O:
\begin{itemize}
    \item We show that RL with segment feedback can be viewed as a discounted finite-horizon MDP, with carefully defined terminal rewards and initial state distribution, and prove that the policy gradient of this finite-horizon MDP is a scaled version of the policy gradient of the original MDP (Theorem \ref{thm:difference_theorem_statement}). This formally justifies that policy gradient can be computed with segment preferences. 

    \item We characterize an estimator error bound for a given segment length (Proposition \ref{prop:L_theorem_statement}), which reveals a fundamental tradeoff in choosing the segment length, which should not be too small or too large. 
\end{itemize}

\noindent{\bf Experimental Evaluations.} We evaluate SP3O empirically against Online DPO \citep{guo2024direct}, P3O \citep{wu2023pairwiseproximalpolicyoptimization}, and ZPG \citep{zhang2025zeroth} in simulated robotic control environments and an LLM finetuning task. In both experiments we show that SP3O outperforms the other algorithms in most settings, with its advantage growing as the horizon grows. We also conduct a small ablation study on the segment length to show that using either a too small or a too large segment length can result in sub-optimal performance, which is consistent with the tradeoff described in Proposition~\ref{prop:L_theorem_statement}.

\section{Preliminaries}

We consider an infinite-horizon discounted Markov Decision Process (MDP) defined by the tuple $\mathcal{M} = (\mathcal{S}, \mathcal{A}, d_0, r, \gamma, P)$, with state space $\mathcal{S}$, action space $\mathcal{A}$, initial state distribution $d_0$, reward function $r(s,a):\mathcal{S} \times \mathcal{A} \to [0,1]$, discount factor $\gamma \in (0,1)$, and transition kernel $P$. 
We consider the reinforcement learning from preference feedback setting, where the agent does not directly observe reward $r(s,a)$ and has to be trained with preference feedback usually generated by humans.

A policy $\pi:\mathcal{S} \to \Delta(\mathcal{A})$ is a function from the state space to a probability distribution on actions. We consider a parametrized policy $\pi_\theta,$
where $\theta$ could be a neural network. 
We define the state and state-action value functions for MDP $\mathcal{M}$ under policy $\pi$ as:
\begin{subequations}
\begin{align}
    V_\mathcal{M}^\pi(s) =& \E_{a_t \sim \pi(\cdot | s_t), s_1 = s} \left [\sum_{t=1}^\infty \gamma^{t-1}r(s_t,a_t) \right], \\
    Q_\mathcal{M}^\pi(s,a) =& \E_{a_t \sim \pi(\cdot | s_t), s_1 = s, a_1 = a} \left [\sum_{t=1}^\infty \gamma^{t-1}r(s_t,a_t)\right],
\end{align} 
\end{subequations} where $s_t$ is the state at time $t$ and $a_t$ is the action taken at time $t.$
We also define the state occupancy measure of policy $\pi$ at timestep $t$ as $d^{\pi}_{\mathcal{M},t}(s) = \prob[s_t = s \mid s_1 \sim d_0, \pi]$, the discounted state occupancy measure of policy $\pi$ as: $d_\mathcal{M}^{\pi}(s) = (1-\gamma) \sum_{t=1}^\infty \gamma^{t-1} d^{\pi}_{\mathcal{M},t}(s)$, and the value of policy $\pi$ as the expected discounted reward as follows,
\begin{align}
    J_\mathcal{M}(\pi) = \E_{s_1 \sim d_0}[V^\pi_\mathcal{M}(s_1)]
\end{align}
Our goal is to find policy $\pi_{\theta^*}$ that maximizes the value, i.e., $\theta^*\in \argmax_\theta J_\mathcal{M}(\pi_\theta)$. 

\subsection{Preference Feedback on Segments}

Given two segments $\sigma^1=\{s_k^1, a_k^1\}_{k=1}^L$ and $\sigma^2=\{s_k^2, a_k^2\}_{k=1}^L$  with length $L$, we assume $\sigma^1$ is preferred over $\sigma^2$ with probability:
\begin{align}\label{eq:preference}
    \prob\left[\sigma^1 \succ \sigma^2\right] =\mathtt{logistic}\Big( \left[R(\sigma^1) + \gamma^{L-1} \widehat{Q}(s_L^1, a_L^1)\right] - \left[R(\sigma^2) + \gamma^{L-1} \widehat{Q}(s_L^2, a_L^2)\right] \Big),
\end{align}
where $R(\sigma) = \sum_{k=1}^{L-1} \gamma^{k-1} r(s_k, a_k)$ is the discounted reward sum of segment $\sigma$ with length $L$ and 
$$\widehat{Q}(s_L,a_L) = Q_{\mathcal{M}}^{\pi_\mathrm{ref}}(s_L,a_L) + \xi$$ is a noisy estimation of $Q_{\mathcal{M}}^{\pi_\mathrm{ref}}(s_L,a_L)$ (the $Q$-value of the last state-action pair of the reference policy) and $\xi\in[-\nu,\nu]$ is a bounded noise that could be adversarial. 

Given an estimate of the preference $\hat{\prob}[\sigma^1 \succ \sigma^2]$, we can use the inverse of the logistic function to recover an estimate $D(\sigma^1, \sigma^2)$ of the human-perceived 
return difference:
\begin{align}\label{eq:D_eq}
 D(\sigma^1, \sigma^2) = \texttt{logistic}^{-1}(\hat\prob[\sigma^1 \succ \sigma^2]) \approx \left[R(\sigma^1) + \gamma^{L-1} \widehat{Q}(s_L^1, a_L^1)\right] - \left[R(\sigma^2) + \gamma^{L-1} \widehat{Q}(s_L^2, a_L^2)\right].
\end{align}

We note that even though the partial return preference model in~\eqref{eq:BT} has been assumed in most of the PbRL/RLHF literature \citep{christiano2017deep, sadigh2017active, ibarz2018reward, lee2021pebblefeedbackefficientinteractivereinforcement, ouyang2022traininglanguagemodelsfollow, bai2022training}, it has been observed that the real preference of humans also depends on the quality of the final state-action pair~\citep{knox2022models} of the segment, which is not captured in the partial return model~\eqref{eq:BT}. To address this \citet{knox2022models} introduced the regret model
\begin{align*}
\prob_\mathrm{regret}[\sigma^1 \succ \sigma^2] = \frac{\exp \left (\sum_{t=1}^L-A^{\pi^*}(s_t^1,a_t^1)\right )}{\exp \left (\sum_{t=1}^L-A^{\pi^*}(s_t^1,a_t^1) \right ) + \exp \left (\sum_{t=1}^L-A^{\pi^*}(s_t^2,a_t^2) \right )},
\end{align*} where $A^{\pi^*}$ is the advantage function of the optimal policy. However, it is possible that the optimal policy is unknown to the human, meaning that they cannot meaningfully estimate this quantity. In contrast, our model only requires the human to be able to approximate the future performance of the policy they are currently viewing, which allows them to extrapolate viewed behavior to future states. For instance, in a goal-reaching task an evaluator would prefer a segment where the agent is moving closer to the goal over a segment where the agent is moving away, even if the average distance to the goal is the same. The partial return model may not capture this difference, while ours can. We provide a further comparison of preference models in Section~\ref{sec:related_work}.

\section{Segment Pairwise Proximal Policy Optimization}

In this section, we first present our algorithm, SP3O (Algorithm~\ref{alg:SP3O}), based on segment preferences. 
Each policy iteration of SP3O consists the following steps:
\begin{itemize}
    \item sample $R$ trajectories from a reference policy $\pi_{\theta_\mathrm{ref}}$;
    \item use these trajectories to sample ${N_p}$ segment pairs $\{(\sigma^1_n,\sigma_n^2)\}_{n=1}^{N_p}$ with a discounted probability distribution (Algorithm~\ref{alg:sample_pairs});
    \item use the preference to generate estimates of the return difference $D(\sigma^1_n, \sigma_n^2)$ for each segment pair and create a preference dataset $\mathcal{D}_{\mathrm{ref}} = \{(\sigma_n^1,\sigma_n^2, D(\sigma^1_n,\sigma_n^2))\}_{n=1}^{N_p}$;
    \item estimate the policy gradient from a loss function (Equation~\ref{eq:loss}) constructed through off-policy importance sampling, and update the current policy $\pi_{\theta_t}$ accordingly.
\end{itemize}
\begin{algorithm*}[htbp]
\caption{Segment Pairwise Proximal Policy Optimization (SP3O)} 
\label{alg:SP3O}
\begin{algorithmic}[1]
\Require initial parameter $\theta_0$, trajectory segment length $L$, preference oracle, number of policy updates $B$, number of policy updates per reference policy $T$, effective horizon $H=C_0/(1-\gamma)$ which is a multiple of $L$, number of reference trajectories $R$, number of trajectory pairs $N_p$, number of human evaluators $M$;
\State $\theta_{\rm{ref}} \gets \theta_0$;
\For{$i = 1 : B$}
    \State sample a batch $\mathcal{T}_{\mathrm{ref},i} = \{\tau_{j}\}_{j=1}^R$ of trajectories with length $H$ using behavior policy $\pi_{\theta_\mathrm{ref}}$;
    \State sample segment pairs $\{(\sigma_n^1,\sigma_n^2)\}_{n=1}^{N_p}$ with Algorithm~\ref{alg:sample_pairs};
    \For{$n=1 : N_p$}
        \State query the preference oracle $M$ times and gather the binary feedback $\{o_{n,m}\}_{m=1}^M$ where \\ \hspace{\algorithmicindent} \hspace{\algorithmicindent}  $\prob[o_{n,m} = 1] = \prob[\sigma_n^1 \succ \sigma_n^2]$ and estimate the preference as $\hat{\prob}[\sigma_n^1 \succ \sigma_n^2]$ as $\frac{\sum_{m=1}^M o_m}{M}$;
    \EndFor
    \State for each segment pair $(\sigma_n^1, \sigma_n^2)$, use $\hat{\prob}[\sigma^1 \succ \sigma^2]$ to obtain $D(\sigma^1_n,\sigma^2_n)$ using Equation~\eqref{eq:D_eq} ;
    \State create preference dataset $\mathcal{D}_{\mathrm{ref}} = \{(\sigma_n^1,\sigma_n^2, D(\sigma^1_n,\sigma_n^2))\}_{n=1}^{N_p}$;
    \For{$t = 1 : T$} 
        \State update $\theta_t$ by minimizing $\mathcal{L}_{\mathrm{SP3O}}(\theta_t; \theta_{t-1}, \theta_\mathrm{ref}, \mathcal{D}_\mathrm{ref,i})$~\eqref{eq:loss};
    \EndFor
    \State $\theta_\mathrm{ref} \gets \theta_T$, $\theta_0 \gets \theta_\mathrm{ref}$;
\EndFor
\State \Return $\theta_T$
\end{algorithmic}
\end{algorithm*}
\begin{algorithm}[ht]
\caption{Sample Segment Pairs} 
\label{alg:sample_pairs}
\begin{algorithmic}[1]
\Require trajectories $\{\tau_j\}_{j=1}^R$ with $\tau_j = \{s_{j,t},a_{j,t}\}_{t=1}^H$, segment length $L$, number of pairs $N_p$;
\State decompose each $\tau_j$ into $\{(s_{j,t},a_{j,t})\}_{t=1}^H$;
\State split trajectories into non-overlapping segments $\psi_{i,j} = \{(s_{i,jL+t},a_{i,jL+t})\}_{t=1}^{L}$ for $1 \leq i \leq R$ and $0 \leq j < H/L$;
\State assign each segment $\psi_{i,j}$ a weight $w_{i,j} \gets \gamma^{L\cdot j}$;
\State sample first segments $\{\sigma_n^1\}_{n=1}^{N_p}$ from $\psi_{i,j}$s according to their weights such that $\psi_{i,j}$ is sampled with probability $\frac{w_{i,j}}{\sum_{i',j'} w_{i',j'}}$;
\State sample second segments $\{\sigma_n^2\}_{n=1}^{N_p}$ from $\psi_{i,j}$s according to their weights such that $\psi_{i,j}$ is sampled with probability $\frac{w_{i,j}}{\sum_{i',j'} w_{i',j'}}$ ensuring that $\sigma_n^1 \neq \sigma_n^2$;
\State \Return $\{(\sigma_n^1,\sigma_n^2)\}_{n=1}^{N_p}$
\end{algorithmic}
\end{algorithm}
Our algorithm uses a loss function with PPO-style clipping regularization defined below
\begin{align}\label{eq:loss}
    \mathcal {L}_\mathrm{SP3O}(\theta_t ; \theta_{t-1}, \theta_\mathrm{ref}, \mathcal{D}_\mathrm{ref}) = - \frac{1}{N_p} \sum_{n=1}^{N_p} \min \Bigg ( & D(\sigma^1_n,\sigma^2_n) \prod_{k=1}^L \frac{\pi_{\theta_t}(a_{n,k}^1|s_{n,k}^1)}{\pi_{\theta_\mathrm{ref}}(a_{n,k}^1|s_{n,k}^1)} \frac{\pi_{\theta_{t-1}}(a_{n,k}^2|s_{n,k}^2)}{\pi_{\theta_\mathrm{ref}}(a_{n,k}^2|s_{n,k}^2)} \nonumber, \\ &D(\sigma^1_n,\sigma^2_n) \prod_{k=1}^L \left ( \frac{\pi_{\theta_t}(a_{n,k}^1|s_{n,k}^1)}{\pi_{\theta_\mathrm{ref}}(a_{n,k}^1|s_{n,k}^1)} \right )^{1+\epsilon}_{1-\epsilon} \left (\frac{\pi_{\theta_{t-1}}(a_{n,k}^2|s_{n,k}^2)}{\pi_{\theta_\mathrm{ref}}(a_{n,k}^2|s_{n,k}^2)} \right )^{1+\epsilon}_{1-\epsilon} \Bigg ) \end{align}
where $(\cdot)^{1+\epsilon}_{1-\epsilon}$ is the clip operator in PPO, i.e., $(\cdot)^{1+\epsilon}_{1-\epsilon} = \min\{\max\{\cdot, 1-\epsilon\}, 1+\epsilon\}$.

We next explain the construction of this loss function. At policy iteration step $t$, the policy gradient can be computed using the policy value difference since:
$$\nabla_{\theta_t} J_{\mathcal{M}}(\pi_{\theta_t}) = \nabla_{\theta_t}(J_\mathcal{M}(\pi_{\theta_t}) - J_\mathcal{M}(\pi_{\theta_{t-1}})).$$
The policy value difference $J_{\mathcal{M}}(\pi_{\theta_t}) - J_{\mathcal{M}}(\pi_{\theta_{t-1}})$ can be evaluated {\em off-policy} using trajectories generated from the reference policy $\pi_{\mathrm{ref}}$ of length $H=C_0/(1-\gamma)$ according to the following lemma. 
\begin{lemma}\label{lemma:pdsurrogate}
Suppose we have three policies $\pi_1,\pi_2,\pi_\mathrm{ref}$. Assume that the support of $\pi_1$ and $\pi_2$ is enclosed within the support of $\pi_\mathrm{ref}$, then a surrogate to estimate $J_{\mathcal{M}}(\pi_1) - J_{\mathcal{M}}(\pi_2)$ is given by the following approximation:
\begin{align}
    &J_\mathcal{M}(\pi_1) - J_\mathcal{M}(\pi_2) \nonumber \\ =& \E_{\substack{a^1_t, a^2_t \sim \pi_\mathrm{ref}, \\ s^1_1, s^2_1 \sim d_0}}\Bigg [ \prod_{t=1}^H\frac{\pi_1(a^1_t|s^1_t)}{\pi_\mathrm{ref}(a^1_t|s^1_t)} \frac{\pi_2(a^2_t|s^2_t)}{\pi_\mathrm{ref}(a^2_t|s^2_t)} \times \left( \sum_{t=1}^H \gamma^{t-1}r(s^1_t,a^1_t) - \sum_{t=1}^H \gamma^{t-1}r(s^2_t,a^2_t) \right)\Bigg] \label{eq:objective-diff} + \lambda,
\end{align}
where $\lambda=\mathcal{O}\left(\gamma^H\right) = \mathcal{O}(e^{-C_0})$ due to the use of an effective horizon and $\tau^1 = \{(s_{t}^1, a_{t}^1)\}_{t=1}^H$ and $\tau^2 = \{(s_{t}^2, a_{t}^2)\}_{t=1}^H$ are independent trajectories. 
\end{lemma}
The proof of this lemma can be found in Section~\ref{sec:obj_diff_proof}. Since the agent does not directly observe the reward $r(s_t,a_t)$, the algorithm relies on preference feedback. In particular, we use $D(\tau^1, \tau^2)$ as an approximation of  $\sum_{t=1}^H \gamma^{t-1}r(s^1_t,a^1_t) - \sum_{t=1}^H \gamma^{t-1}r(s^2_t,a^2_t).$ 
Note that our algorithm uses segments of length $L$ instead of full trajectories with length $H$ because a comparison of long trajectories could be time-consuming for the evaluators. For example, in autonomous driving, it is hard for a human evaluator to compare trajectories that are hours-long. 
In Section~\ref{sec:theory}, we show that under mild assumptions, the policy gradient estimated from segment preferences is the same as the policy gradient estimated from full trajectory preferences, up to a constant scaling factor. 

Now, given a set of segment pairs generated by the reference policy $\{(\sigma_n^1,\sigma_n^2)\}_{n=1}^{N_p}$, where $\sigma_n^1 = \{(s_{n,k}^1, a_{n,k}^1)\}_{k=1}^L$ and $\sigma_n^2 = \{(s_{n,k}^2, a_{n,k}^2)\}_{k=1}^L$, we can use the preference dataset to estimate the policy difference as:
\begin{align*}
    J_\mathcal{M}(\pi_{\theta_t}) - J_\mathcal{M}(\pi_{\theta_{t-1}}) \approx \frac{1}{N_p} \sum_{n=1}^{N_p} D(\sigma^1_n, \sigma^2_n) \prod_{k=1}^L \frac{\pi_{\theta_t}(a_{n,k}^1|s_{n,k}^1)}{\pi_\mathrm{ref}(a_{n,k}^1|s_{n,k}^1)} \frac{\pi_{\theta_{t-1}}(a_{n,k}^2|s_{n,k}^2)}{\pi_\mathrm{ref}(a_{n,k}^2|s_{n,k}^2)}, \label{eq:objective-diff-est}
\end{align*}
without using any trajectories generated by $\pi_{\theta_t}$ or $\pi_{\theta_{t-1}}$
To further facilitate stability, we plug this estimator into the PPO-type clip mechanism~\citep{schulman2017proximalpolicyoptimizationalgorithms, wu2023pairwiseproximalpolicyoptimization}, which leads to  
the loss function found in Equation~\eqref{eq:loss}. Our loss function generalizes the P3O \citep{wu2023pairwiseproximalpolicyoptimization} loss function, originally designed under the contextual bandit model, to the MDP setting with segments-level preferences.

{\bf Remark.} We may consider variants of this loss function. For example, a KL-penalty term between $\pi_{\theta}$ and $\pi_{\theta_\mathrm{ref}}$ would also be a reasonable choice for stability instead of clipping. 
This KL regularized loss function is shown below
\begin{align*}
    & \mathcal {L}_\mathrm{SP3O-KL}(\theta_t ; \theta_{t-1}, \theta_\mathrm{ref}, \mathcal{D}_\mathrm{ref}) \nonumber \\ = & \beta D_{\mathrm{KL}}(\pi_{\theta_t}\|\pi_{\theta_\mathrm{ref}})- \frac{1}{N_p} \sum_{n=1}^{N_p}
    D(\sigma^1_n,\sigma^2_n) \times \prod_{k=1}^L \left ( \frac{\pi_{\theta_t}(a_{n,k}^1|s_{n,k}^1)}{\pi_{\theta_\mathrm{ref}}(a_{n,k}^1|s_{n,k}^1)} \right ) \left (\frac{\pi_{\theta_{t-1}}(a_{n,k}^2|s_{n,k}^2)}{\pi_{\theta_\mathrm{ref}}(a_{n,k}^2|s_{n,k}^2)} \right ).
\end{align*}
In Section~\ref{sec:kl_experiments}, we present an empirical comparison of the two methods in simulated robotic control environments and show that KL regularization does not perform better than clipping in that setting. Furthermore, one may replace $\pi_{\theta_{t-1}}$ with any choice of policy $\pi$ sufficiently similar to $\pi_{\theta_\mathrm{ref}}$, such as $\pi_{\theta_\mathrm{ref}}$ itself or a perturbed version of $\pi_{\theta_{t-1}}$. SP3O uses $\pi_{\theta_{t-1}}$ with the intuition that the loss function should have less emphasis on behaviors that have already been learned or corrected. This freedom to choose which policy to compare against is a further generalization of P3O, which uses $\theta_{t-1} = \theta_t$. In Section~\ref{sec:second_policy_choice} we present empirical results replacing $\pi_{\theta_{t-1}}$ with $\pi_{\mathrm{ref}}$ itself and show our design indeed performs better than using $\pi_{\theta_\mathrm{ref}}$.

\section{Theoretical Analysis}\label{sec:theory}

This section presents two theoretical results. The first shows that under mild assumptions, the estimator bootstrapped from segments is an unbiased estimator of the gradient of the original objective in Equation~\eqref{eq:objective-diff}. The second quantifies the tradeoff in choosing $L,$ the segment length. 

\subsection{Policy Gradient Estimation from Trajectory Segments}
Given a policy $\pi_\mathrm{ref}$, we construct the following discounted finite horizon MDP $\mathcal{M'}$, which we call the \emph{segment MDP}. It models trajectory segments generated from $\pi_\mathrm{ref}$ as follows:
\begin{align} \mathcal{M}' = (\mathcal{S},\mathcal{A}, d_0', r_h', P, \gamma, L), \nonumber \text{ with }
r_h'(s,a) = \begin{cases}
r(s,a), & h \neq L; \\
Q_\mathcal{M}^{\pi_\mathrm{ref}}(s,a), & h = L. \nonumber
\end{cases}\end{align}
where $L$ is the trajectory-segment length and the initial distribution $d_0'(s)$ is the discounted occupancy measure defined as follows:
\begin{align*}
d_0'(s) = \left(1-\gamma^L\right)\sum_{k=0}^\infty \gamma^{kL} \prob\left[s_{kL + 1} = s \mid s_1 \sim d_0, \pi_\mathrm{ref}\right].
\end{align*}
The following theorem shows that given a reference policy close to $\pi_\theta$, the gradient of the expected policy value difference in $\mathcal{M'}$ is a good estimate for the gradient of the objective in $\mathcal{M}$. 

\begin{theorem}\label{thm:difference_theorem_statement}
If $\pi_{\mathrm{ref}} = \pi_\theta$, then for any policy $\pi',$ the following holds
\begin{align*}
    \nabla_\theta J_\mathcal{M}(\pi_\theta) = \frac{1}{1-\gamma^L} \nabla_{\theta}(J_{\mathcal{M'}}(\pi_\theta) - J_{\mathcal{M'}}(\pi')).
\end{align*} In other words, the gradient of the expected policy value in $\mathcal{M}$ is a scaled version of the gradient of the expected policy value difference in $\mathcal{M'}$. 
\end{theorem}
The complete proof of the theorem is in Section~\ref{sec:theorem1proof}, which relies on two intermediate results. The first shows that with respect to $\pi_\mathrm{ref}$, the $Q$ function in the segment MDP $\mathcal{M'}$ equals the $Q$ function in the original MDP $\mathcal{M}$, i.e., $Q_{\mathcal{M'}}^{\pi_\mathrm{ref}}(s,a) = Q_{\mathcal{M}}^{\pi_\mathrm{ref}}(s,a)$ for all state-action pairs.
This is a direct result of the definition of the terminal reward of $\mathcal{M'}$ being the $Q$ function in $\mathcal{M}$. 
The second intermediate result shows that the discounted state occupancy measure of policy $\pi_\mathrm{ref}$ in the segment MDP $\mathcal{M'}$ is the same as the discounted state occupancy measure of $\pi_\mathrm{ref}$ in the original MDP $\mathcal{M}$. This is because the initial state distribution of $\mathcal{M'}$ is based on the discounted state occupancy of $\pi_\mathrm{ref}$ in $\mathcal{M}$.
The theorem's proof then expands each side of the equation with the policy gradient theorem \citep{sutton1998reinforcement}, and then interchanges the discounted state occupancy measures and $Q$ functions.

\subsection{Trade-offs in Choosing the Segment Length}\label{sec:L_analysis}

To understand the choice of segment length, we consider the segment MDP $\mathcal{M}'$ which generates trajectories of length $L$. 
Suppose that at each policy update we can collect $N$ total timesteps of segments, mimicking an oracle that views $N$ timesteps per update. This yields a dataset $\mathcal{D}$ of segments with size $|{\cal D}| = \frac{N}{2L}$. Assume $\mathcal{D} = \{\sigma^1, \sigma^2, \cdots, \sigma^{|{\cal D}|}\}$ and each $\sigma^n = \{s_k^n, a_k^n\}_{k=1}^L$ are independently sampled. We use the sample average to estimate the expected policy value in the segment MDP $J_{\mathcal{M'}}(\pi_\theta)$
as follows:
\begin{align*}
    \widehat{J}_{\mathcal{M'}}(\pi_\theta, \mathcal{D}) = \frac{1}{|\mathcal{D}|} \sum_{i=1}^{|\mathcal{D}|} \left(R(\sigma^i) + \gamma^{L-1} \widehat{Q}(s_L^i, a_L^i) \right).
\end{align*}
Given this, we have the following bound on the estimation error. 
\begin{proposition}\label{prop:L_theorem_statement}
With probability $1-\delta$, the estimation error is bounded as follows:
    \begin{align*}
        \left|\widehat{J}_{\mathcal{M'}}(\pi_\theta, \mathcal{D}) - J_{\mathcal{M'}}(\pi_\theta)\right| \leq \frac{1}{1-\gamma} \sqrt{\frac{L\log(2/\delta)}{N}} + \gamma^{L-1} \nu.
    \end{align*}
\end{proposition}
The proposition is obtained by applying Hoeffding's inequality to the sample average estimate, and the complete proof is in Section~\ref{sec:L_proof}. It can be observed that if the segment length $L$ is too short, e.g., $L=1$, then we suffer a constant estimation error $\nu$ resulting from the inaccurate human estimation of the final state $Q$ function. This reflects our assumption that the preference oracle has an incomplete understanding of both the policy and the environment, so they cannot tell the agent directly which action is best at each state. On the other hand, when $L$ increases, the preference oracle error $\xi$ gets discounted, but the variance of the partial reward increases with a smaller number of segments $|\mathcal{D}|= \left \lfloor \frac{N}{2L}  \right \rfloor$. This results in a higher error in evaluating the policy from segments. The best segment length $L$ strikes a balance between the two sources of error. 
In the next section, we show that the trade-off indeed appears in practice with intermediary values of $L$ performing best in certain environments.

\section{Experiments}

This section compares SP3O to other reward-model-free PbRL/RLHF algorithms in a suite of robotic control tasks from the Gymnasium benchmark \citep{towers2024gymnasium}, and an LLM finetuning task. In the robotic control tasks, we compare SP3O to Online DPO \citep{guo2024direct}, P3O \citep{wu2023pairwiseproximalpolicyoptimization}, and ZPG \citep{zhang2025zeroth}. In the LLM finetuning task, we compare SP3O to Online DPO and P3O. 

\subsection{Robotic Control Tasks}

\textbf{Experiment Details.} We compare SP3O with three different values of the segment length ($L=5,20,50$) to Online DPO, P3O, and ZPG. Each algorithm is used to train a randomly initialized policy for 100 updates. At each policy update, all algorithms sample 10 trajectories from the environment. This yields 45 trajectory pairs for Online DPO and P3O, and 25 for ZPG. For SP3O, we randomly sample segment pairs so that the total number of timesteps in the pairs is the same as if trajectory pairs had been used. This ensures that SP3O is not giving more information to the preference oracle than Online DPO or P3O. 

For each pair of segments/trajectories $(\sigma^1,\sigma^2)$ or $(\tau^1,\tau^2)$ we simulate human preference as follows. We first use the real environmental reward to compute $\prob[\sigma^1 \succ \sigma^2]$ or $\prob[\tau^1\succ\tau^2]$. For SP3O, we set the $\hat{Q}$ term equal to the discounted reward sum of the remainder of the sampled rollout. For all algorithms we use the process described in Algorithm~\ref{alg:SP3O} with $M=50$ to obtain an estimate $\hat{\prob}[\sigma^1 \succ \sigma^2]$ or $\hat{\prob}[\tau^1\succ \tau^2]$ of the human preference. Further details are in Section~\ref{sec:implementation_details_rob}. 

\textbf{Performance Comparison.} We compare each algorithm's final policy value with the environment's time horizon. We normalize the final policy value with the time horizon, reporting average reward per timestep. Because we are showing how performance varies with time horizon, we exclude the MuJoCo environments where termination before the horizon is reached is very common. The results, averaged across 100 seeds per point, are shown in Figure~\ref{fig:main_results}, where the area between the error bars represents 95\% confidence intervals. 

\begin{figure}[htbp]
    \centering
    \begin{subfigure}[b]{0.32 \linewidth}
        \centering
        \includegraphics[width=1.0\textwidth]{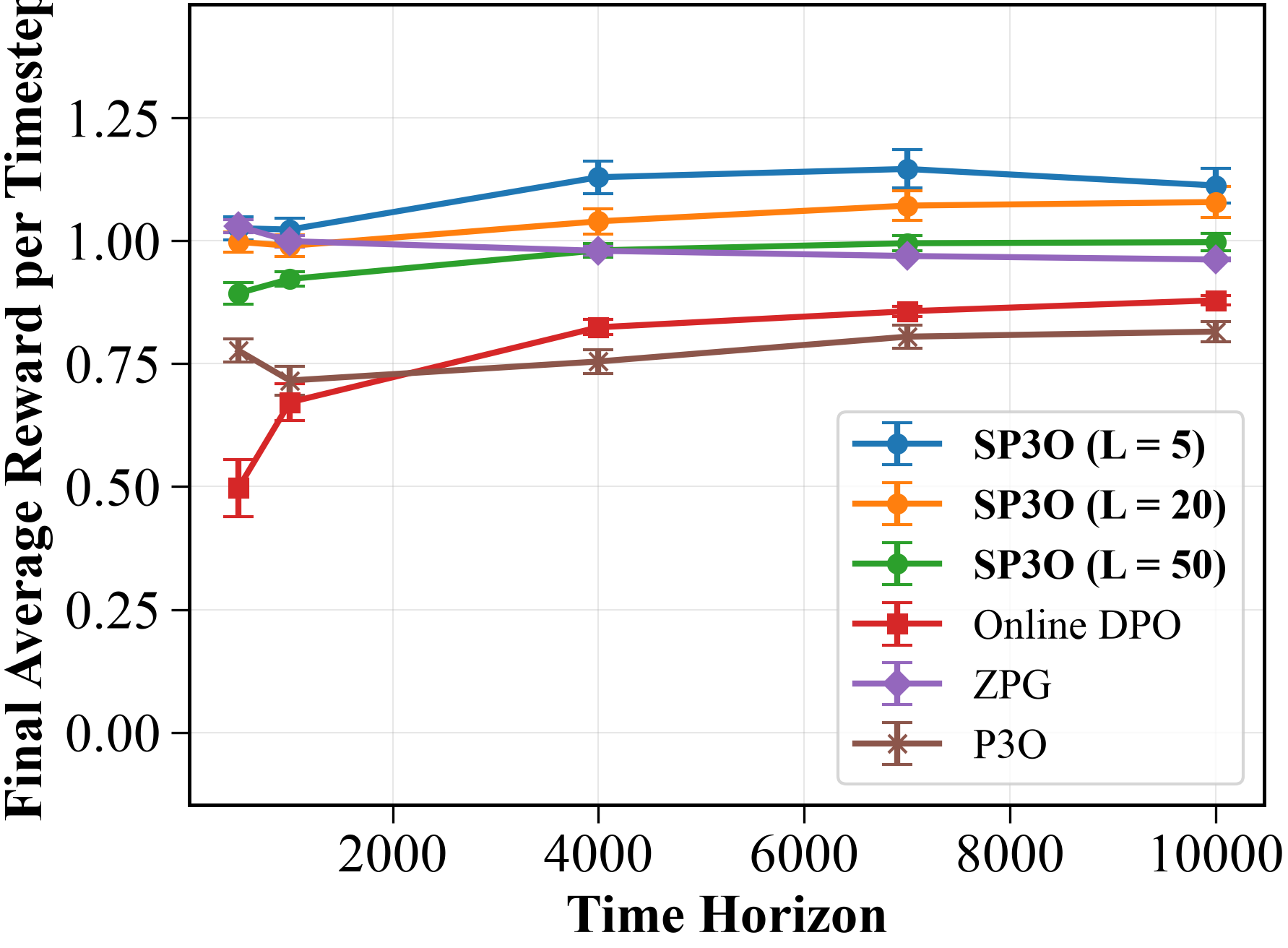} 
        \caption{Ant-v5}
    \end{subfigure}
    \hfill
    \begin{subfigure}[b]{0.32 \linewidth}
        \centering
        \includegraphics[width=1.0\textwidth]{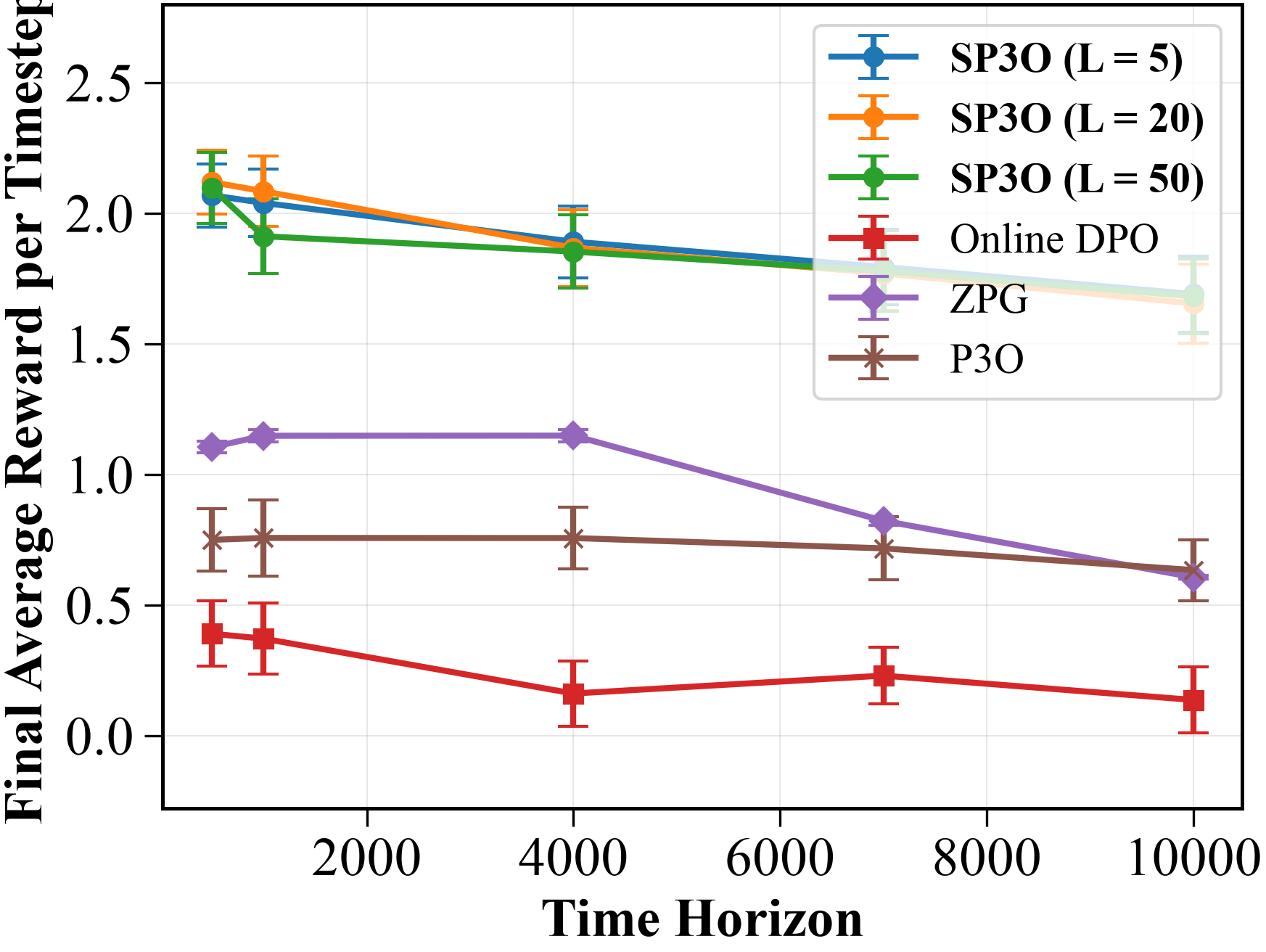} 
        \caption{HalfCheetah-v5}
    \end{subfigure}
    \hfill
    \begin{subfigure}[b]{0.32 \linewidth}
        \centering
        \includegraphics[width=1.0\textwidth]{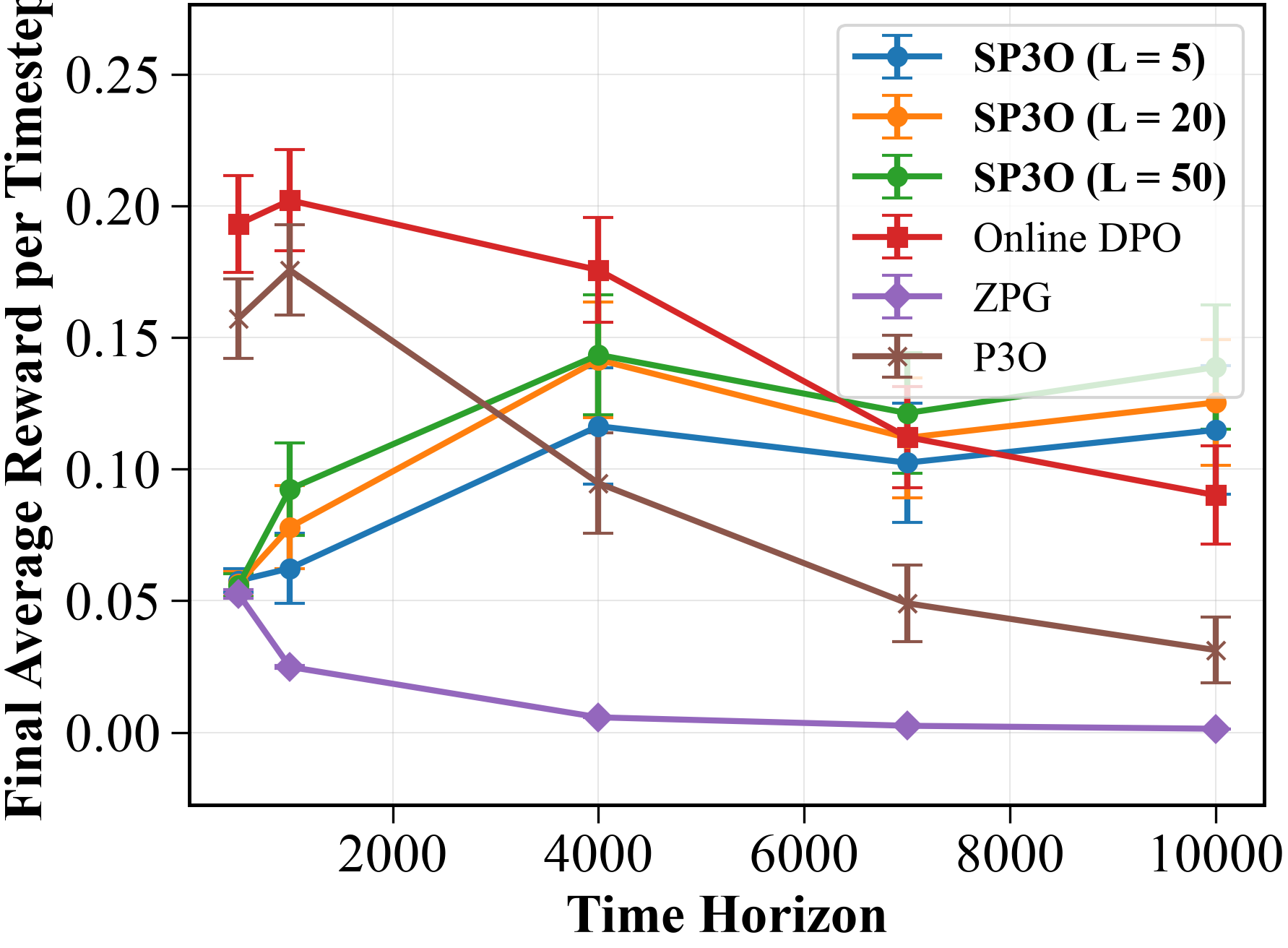} 
        \caption{Swimmer-v5}
    \end{subfigure}
    \caption{Final Average Reward per Timestep versus Time Horizon in MuJoCo Environments.}
    \label{fig:main_results}
\end{figure}

In Ant-v5, SP3O and Online DPO improve as the time horizon grows, P3O stays flat, and ZPG degrades. In HalfCheetah-v5, SP3O and P3O's performance stays constant while Online DPO and ZPG's decays. In Swimmer-v5, SP3O improves with horizon length, while Online DPO, P3O, and ZPG all degrade. Overall, this suggests that SP3O's segment preferences become more effective as the horizon grows, while the performance of trajectory preferences in the other algorithms degrades. Furthermore, in all environments SP3O with at least one value of $L$ outperforms the other algorithms at the longest time horizon.

\textbf{Segment Length.} We also empirically show the trade-off of segment length $L$ discussed in Section~\ref{sec:L_analysis}. We use the same setup as the previous experiment, with the time horizon set to 1000 in HalfCheetah-v5, and vary the trajectory segment length $L$ and the number of timesteps that can be sent to the oracle, which is equivalent to $N$ in Proposition~\ref{prop:L_theorem_statement}. In Figure~\ref{fig:L_ablation_results} we show, for each of six oracle budgets ($2^{-4}$ to $2^1$ times the $N$ in the main experiment), the final policy value at $L=5,10,\dots,40$ relative to the seed mean, averaged over 4000 seeds. The best performing $L$s in each facet are starred, and the bars are 95\% confidence intervals. 
We observe that as we increase $N$ the best values of $L$ increase monotonically, which is consistent with Proposition~\ref{prop:L_theorem_statement}. A larger $N$ shrinks the first error term $\frac{1}{1-\gamma} \sqrt{L\log(2/\delta)/N}$, so a large $L$ can be afforded to drive down the constant error term $\gamma^{L-1}\nu$ stemming from the oracle's estimate of the $Q$ function.

\begin{figure}[htbp]
    \centering
    \includegraphics[width=\linewidth]{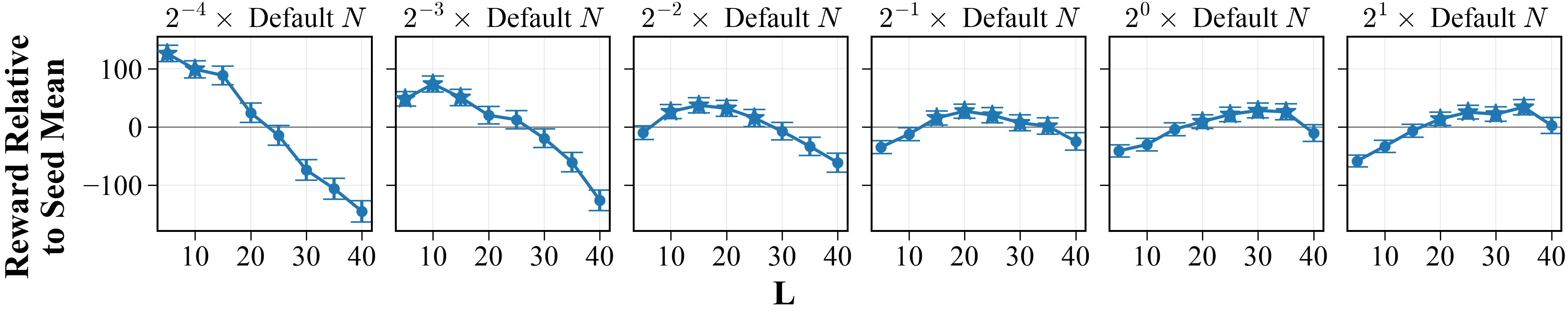}
    \caption{Ablation Study of $L$ and $N$ in HalfCheetah-v5.}
    \label{fig:L_ablation_results}
\end{figure}

\subsection{LLM Finetuning Task}

\textbf{Experiment Details.} In this experiment, we compare SP3O to Online DPO \citep{guo2024direct} and P3O \citep{wu2023pairwiseproximalpolicyoptimization} in an LLM finetuning task. We evaluate both a clipping and KL divergence regularized version of SP3O. For our base model we use GPT-J 6B \citep{gpt-j}, and finetune a rank 16 LoRA \citep{hu2022lora} adapter. The goal of the task is to minimize the toxicity of the language model's outputs while keeping the KL divergence between the fine-tuned and base models small. We use the Real Toxicity Prompts \citep{gehman2020realtoxicityprompts} dataset to prompt the base model to have a toxic output. We measure toxicity with the pretrained Detoxify \citep{Detoxify} model that outputs toxicity values between 0 and 1.

During each training step, each algorithm generates 2 continuations with the current policy (with max length of 64 tokens) for each of 8 randomly selected prompts. For Online DPO and P3O these continuations are turned directly into 8 pairs. To create segments for SP3O, we split each continuation into its sentences. We then randomly pair these sentences such that each sentence is used at most once, except when there are an odd number of sentences, in which case one sentence will be in two pairs. We note that a sentence can be paired with another from the same continuation. With this setup, an evaluator will have to view a similar number of tokens at each training step regardless of whether sentences or full continuations were used. A sentence/continuation is considered preferred over another if the Detoxify score for the first sentence/continuation is at least $0.05$ lower than the score for the other. All pairs that meet this threshold are then passed to the respective loss function. Further details and hyperparameters are in Section~\ref{sec:implementation_details_llm}.

\textbf{Performance Comparison.} Shown in Figure~\ref{fig:LLM_results} are the results after running all three algorithms for 2000 steps with checkpoints evaluated every 50 steps, with 5 different seeds. At each checkpoint we report the mean toxicity on a held-out challenging subset of the prompts compared to the reverse KL divergence $D_{\mathrm{KL}}(\pi \| \pi_{\mathrm{base}})$ between the current and base models. We approximate the challenging mean toxicity versus KL curve with a monotonic PAVA regression. The shaded areas represent 95\% confidence intervals, and we also show a zoomed version of this plot.
We observe that both versions of SP3O consistently outperform Online DPO, and that the performance of SP3O and P3O are very similar. However, in the next section we show that growing the horizon negatively effects P3O and positiely effects SP3O.

\begin{figure}[htbp!]
    \centering
    \begin{subfigure}[b]{0.31 \linewidth}
        \centering
        \includegraphics[width=\linewidth]{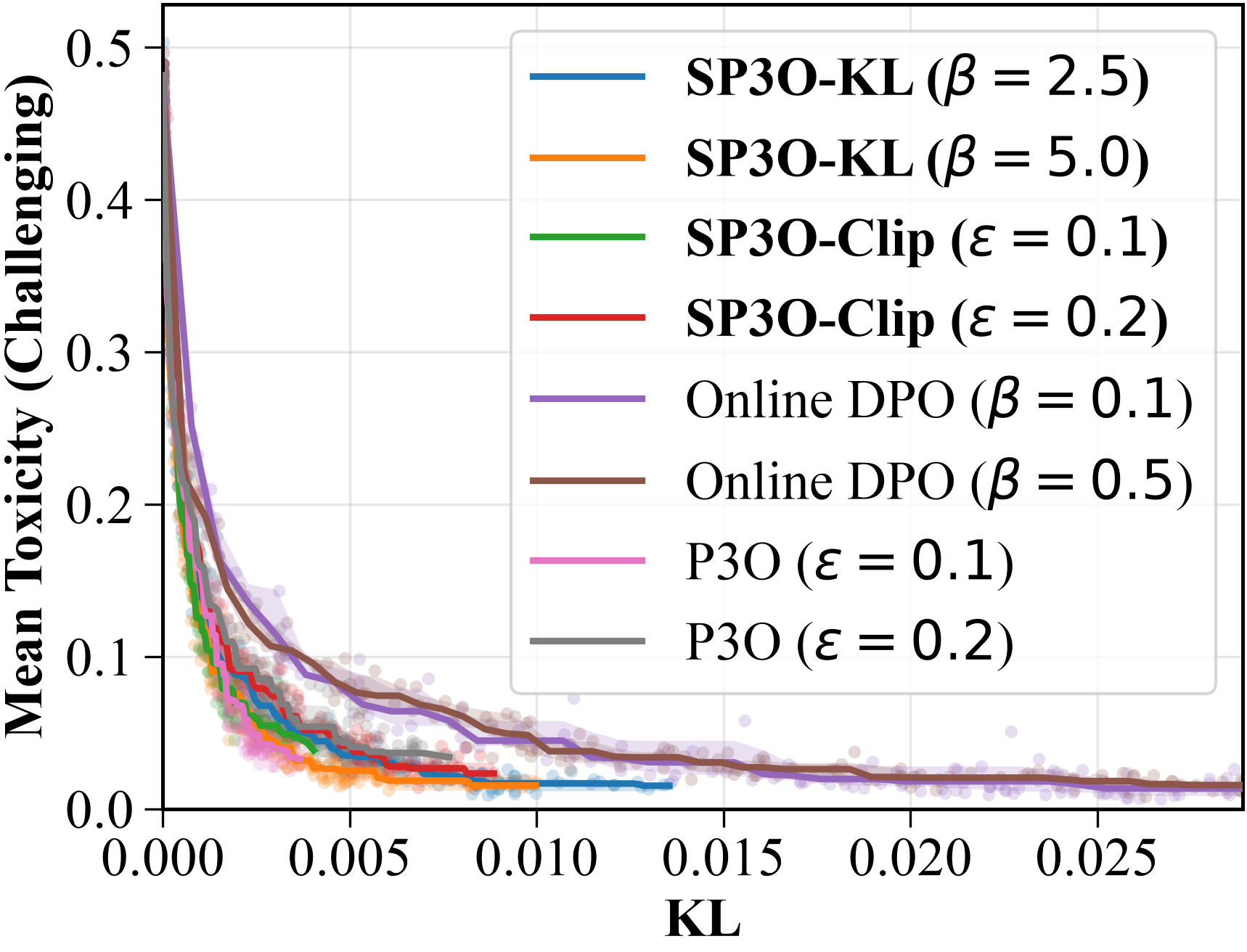}
    \end{subfigure}
    \begin{subfigure}[b]{0.33 \linewidth}
        \centering
        \includegraphics[width=\linewidth]{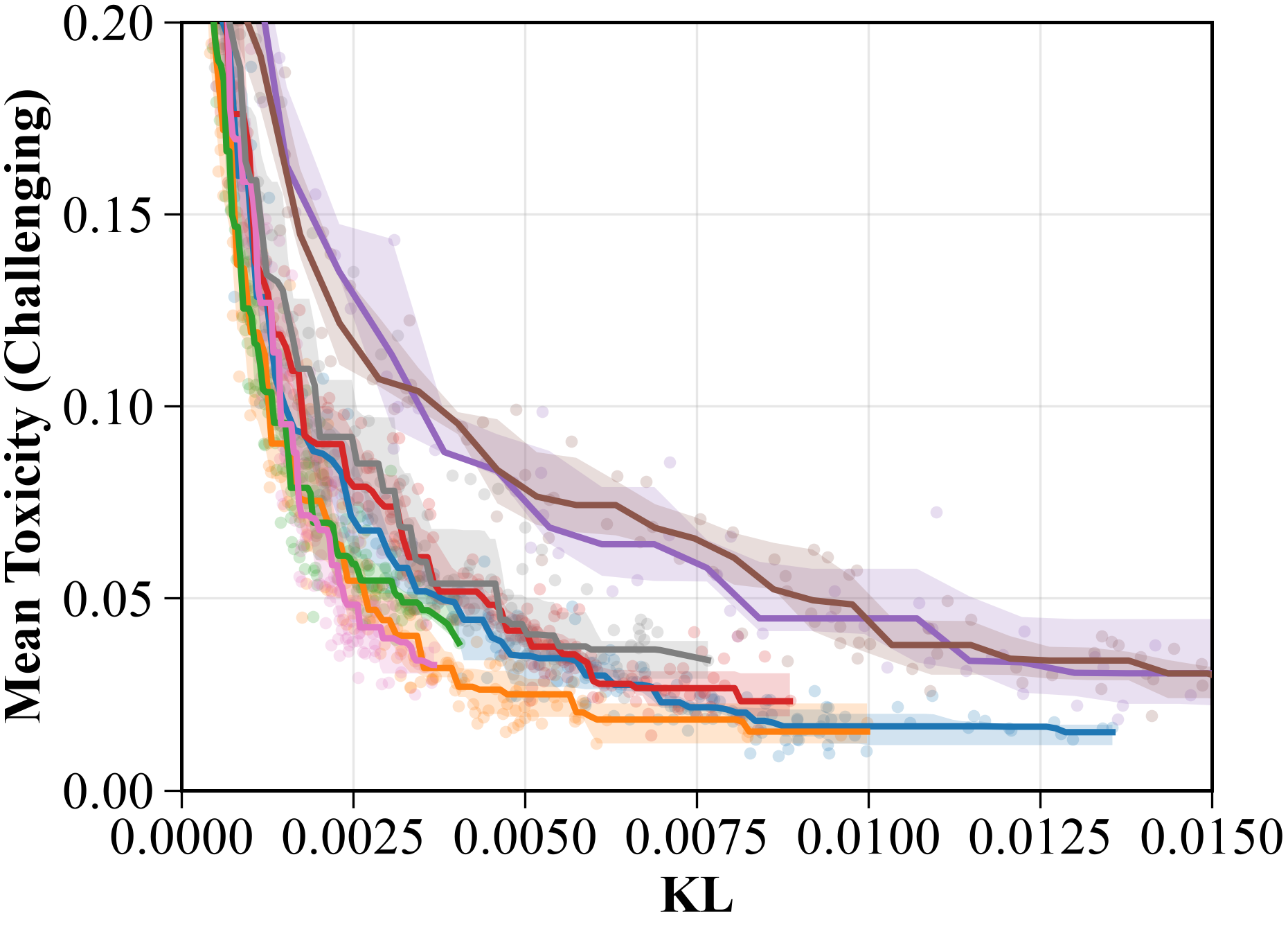}
    \end{subfigure}
    \caption{LLM Finetuning Results.}
    \label{fig:LLM_results}
\end{figure}

\textbf{Continuation Length's Effect on Performance.} We also examine how the number of tokens generated in each online continuation (TPC) affects these algorithms. In this setting, TPC can be thought of as the time horizon. Shown in Figure~\ref{fig:LLM_length} are the results when SP3O, P3O, and Online DPO are used on this task with 64, 128, and 256 tokens per continuation (TPC) averaged over 5 seeds. To make the results comparable all checkpoints were evaluated with 64 tokens per continuation. 

\begin{figure}[htbp]
    \centering
    \includegraphics[width=\linewidth]{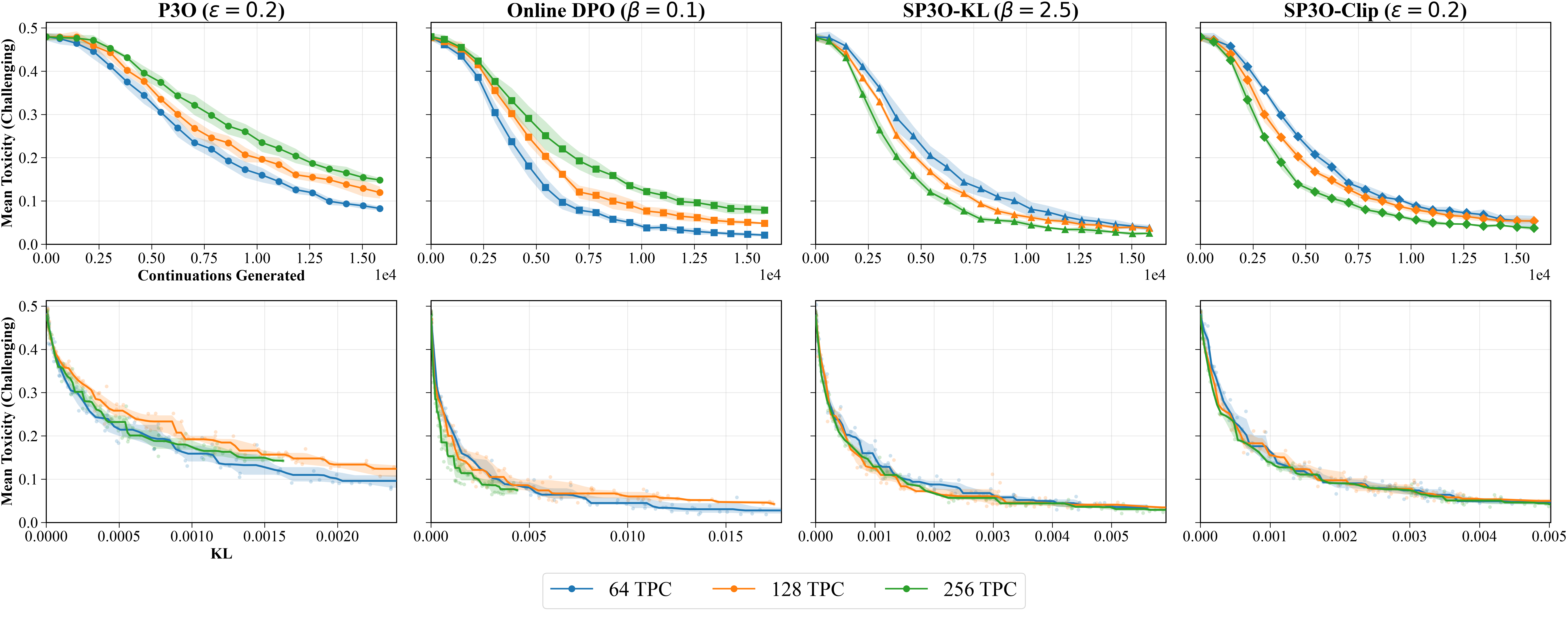}
    \caption{Comparison of Continuation Length.}
    \label{fig:LLM_length}
\end{figure}

We observe that with more tokens per continuation SP3O achieves a low toxicity score quicker. However as TPC increases, Online DPO and P3O exhibit a slower decrease in toxicity. This is despite all algorithms maintaining similar toxicity versus KL curves. As TPC increases more pairs are created for SP3O, despite the total number of tokens in the pairs being the same as if Online DPO or P3O were used. This suggests that using segment preferences becomes more efficient than trajectory preferences as we increase the horizon length.

\section{Conclusion}

In this paper, we introduced SP3O, a novel reward-model-free and critic-free PbRL algorithm applicable to general MDPs with preferences based on trajectory segments. 
We analyzed the use of trajectory segments, showing the unbiasedness of the gradient estimator, and discovered a trade-off in segment lengths. We evaluated SP3O experimentally against other reward-model-free algorithms in a suite of robotic control tasks, and a LLM finetuning task.

\textbf{Limitations and Future Directions.} Like other non-reward modeling PbRL/RLHF algorithms, SP3O requires extensive preference queries throughout the entire training process. Thus, its applications are restricted to scenarios where feedback is easy to obtain. To mitigate this limitation, a rule-based oracle or AI feedback \citep{guo2024direct} could be used. A preference model such as a preference transformer \citep{kim2023preference} or an LSTM preference model \citep{arjona2019rudder, early2022non} could also be used to lower the amount of human feedback required, although this approach may also suffer from issues similar to those of a reward model. Another option to avoid this limitation is to actively select queries that are likely to result in a high reward difference, such as similar strategy was used by \citet{christiano2017deep}, where multiple reward models were trained in parallel, and segments with high disagreement between the models would be selected for human evaluation.

\section*{AI Use Disclosure} AI tools were used in the creation of this work only for writing portions of the experimental code and final polish and proofreading of the text.

\bibliography{references}
\bibliographystyle{unsrtnat}

\newpage
\appendix

\section*{Appendix Table of Contents}
\startcontents[chapter]
\printcontents[chapter]{l}{0}{\setcounter{tocdepth}{2}}

\section{Related Work}\label{sec:related_work}

\subsection{Preference-Based RL.} In the general MDP context, learning from preferences has long been studied under the umbrella of preference-based RL (PbRL) \citep{akrour2011preference, wilson2012bayesian, wirth2017survey}. In PbRL, not only trajectories and trajectory segments can be used for feedback, but also states and actions themselves \citep{wirth2017survey}. Many of the methods prior to PPO+RM \citep{christiano2017deep} are based on restrictive assumptions, including discrete state/action spaces, a fully parametric policy, or a known transition kernel. This makes them inapplicable to many modern RL problems. 

\subsection{Comparison of Preference Models.} For two segments $\sigma^1$ and $\sigma^2$ with length $L$, most existing works on PbRL/RLHF~\citep{christiano2017deep, sadigh2017active, ibarz2018reward} assume the oracle's preferences are based on the undiscounted return of the trajectory segment alone (partial return model), which is useful when training a reward model given the step-wise reward decomposition form. Given two segments $\sigma^1 = \{s_t^1,a_t^1\}_{t=1}^L$ and $\sigma^2 = \{s_t^2,a_t^2\}_{t=1}^L$. The preference is expressed as:
\begin{align*}
\prob_\mathrm{partial}[\sigma^1 \succ \sigma^2] = \frac{\exp \left (\sum_{t=1}^L r(s_t^1,a_t^1) \right )}{\exp \left (\sum_{t=1}^L r(s_t^1,a_t^1) \right ) + \exp \left (\sum_{t=1}^L r(s_t^2,a_t^2) \right )}
\end{align*}
This kind of equation is useful when training a reward model, as it can be directly estimated by the reward model for each segment pair.
However, it has also been challenged~\citep{knox2022models} as not accurately representing how real human preferences are generated. Specifically, even if the partial returns of two segments are the same, humans could still prefer one trajectory segment based on the value of the last state and action. For example, in goal-reaching tasks, an ending state closer to the goal is usually more preferred than a farther ending state even if the partial returns are the same. To accommodate this modeling flaw, \citet{knox2022models} proposed the regret-based preference model, which assumes preference oracles could make decisions based on the cumulative advantage function of each state-action pair in the segment. Specifically, \citet{knox2022models} defines the regret of a segment as:
\begin{align*}
\texttt{regret}(\sigma^1) = \sum_{t=1}^{L} V^{\pi^*}(s_t^1) - Q^{\pi^*}(s_t^1,a_t^1) = \sum_{t=1}^L-A^{\pi^*}(s_t^1,a_t^1),
\end{align*}
where $A$ is the advantage function. Assuming the Bradley-Terry model, we get the following equation for preference:
\begin{align*}
\prob_\mathrm{regret}[\sigma^1 \succ \sigma^2] = \frac{\exp \left (\sum_{t=1}^L-A^{\pi^*}(s_t^1,a_t^1)\right )}{\exp \left (\sum_{t=1}^L-A^{\pi^*}(s_t^1,a_t^1) \right ) + \exp \left (\sum_{t=1}^L-A^{\pi^*}(s_t^2,a_t^2) \right )}.
\end{align*}
The main assumption of the regret model is that the evaluators know or can make an estimate of the advantage function. This means that the human not only has an unknown latent reward function, but also has an understanding of all the environmental transition dynamics and thus can estimate this function. However, in more complicated environments, this assumption could break down.

Our model is between the partial return and regret reward models. Our model does include a version of the $Q$ function, eliminating the main problem with the partial reward model presented by \citet{knox2022models}. However, our model only assumes that the human has an understanding of the transition dynamics that arise from the policy it is evaluating. This could potentially be a much smaller subset than the knowledge of transition dynamics assumed by the regret model. Furthermore, when viewing the segments, it is possible for the human to gain this understanding and extrapolate to future states and actions. Thus, our model is able to avoid the disadvantages of the partial reward model without all the assumptions of the regret model.

\subsection{Use of Trajectory Segments outside of PbRL/RLHF.}
As previously discussed, trajectory segments are commonly used in PbRL/RLHF for MDPs \citep{christiano2017deep, ibarz2018reward, knox2022models, hejna2024contrastivepreferencelearninglearning}. However, their efficacy over full-length trajectories has not been justified or theoretically examined. Outside of PbRL/RLHF, \citet{du2025reinforcement} studies a setting where trajectories can be split into segments and rewards are delayed until the ends of each segment. 

\section{Additional Experiments}

\subsection{KL Regularization vs Clipping}\label{sec:kl_experiments}

In this section, we compare the version of SP3O that uses PPO-style clipping for stability versus a version that uses KL regularization and a version that uses no regularization. Results comparing these two versions of the algorithm are shown in Figure \ref{fig:kl_study}, with the time horizon set to 1000. The curves are smoothing with a sliding window of 50 updates. We observe that in Ant-v5 and HalfCheetah-v5 clipping outperforms both KL and no regularization. In Swimmer-v5, all are comparable considering the confidence intervals. 

\begin{figure}[htbp]
    \centering
    \begin{subfigure}[b]{0.32 \linewidth}
        \centering
        \includegraphics[width=\linewidth]{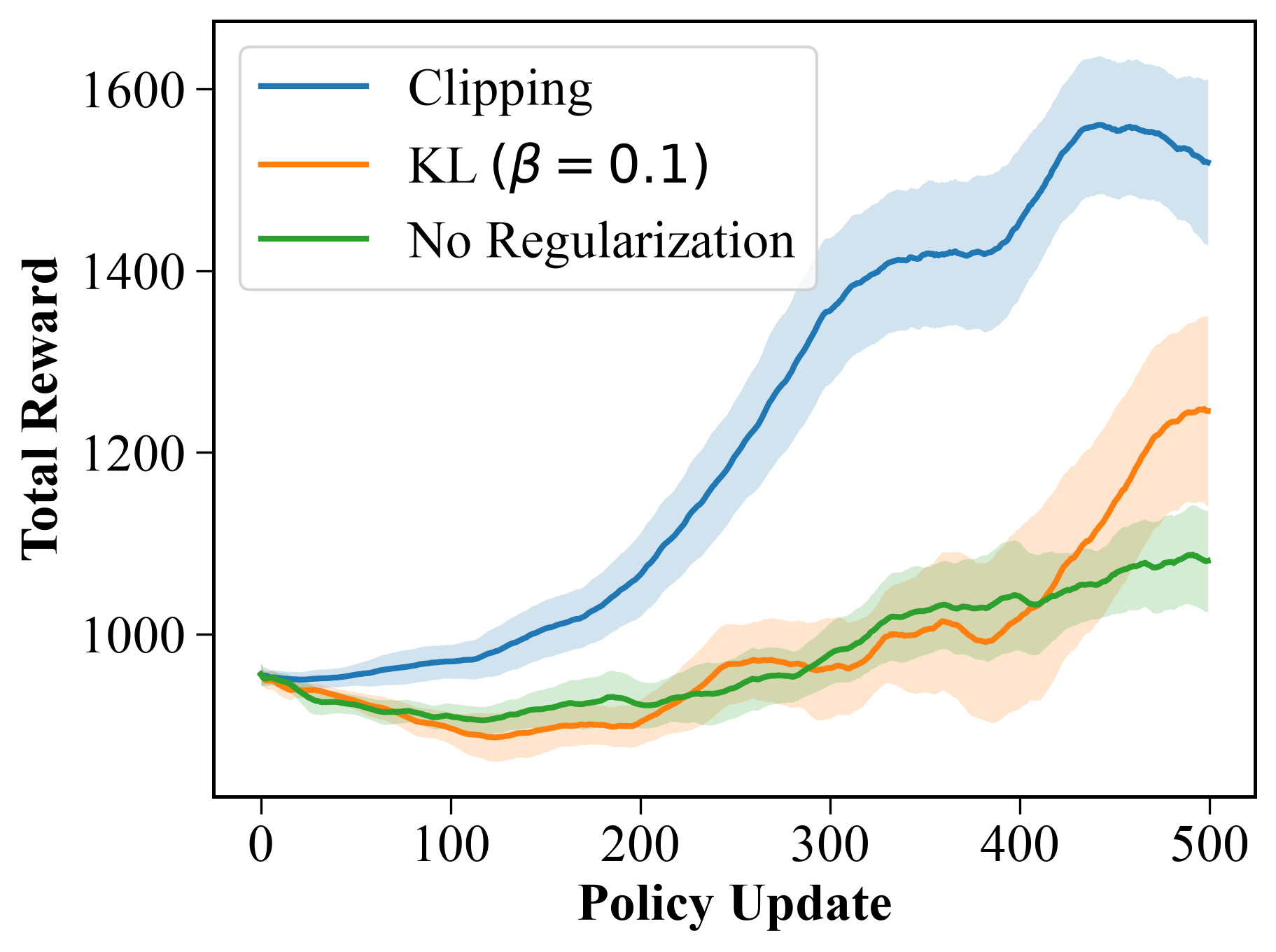} 
        \caption{Ant-v5}
    \end{subfigure}
    \hfill
    \begin{subfigure}[b]{0.32 \linewidth}
        \centering
        \includegraphics[width=\linewidth]{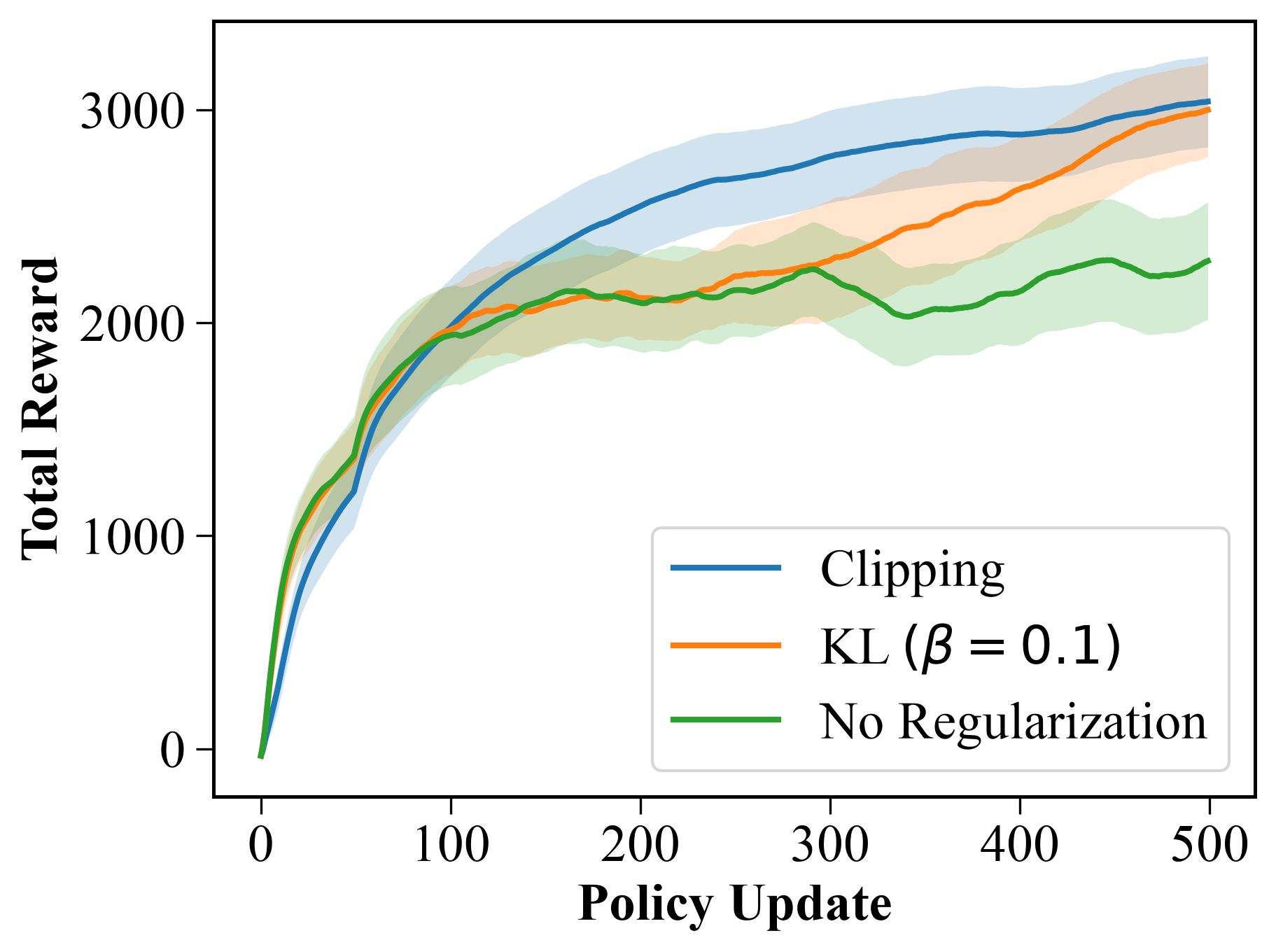} 
        \caption{HalfCheetah-v5}
    \end{subfigure}
    \hfill
    \begin{subfigure}[b]{0.32 \linewidth}
        \centering
        \includegraphics[width=\linewidth]{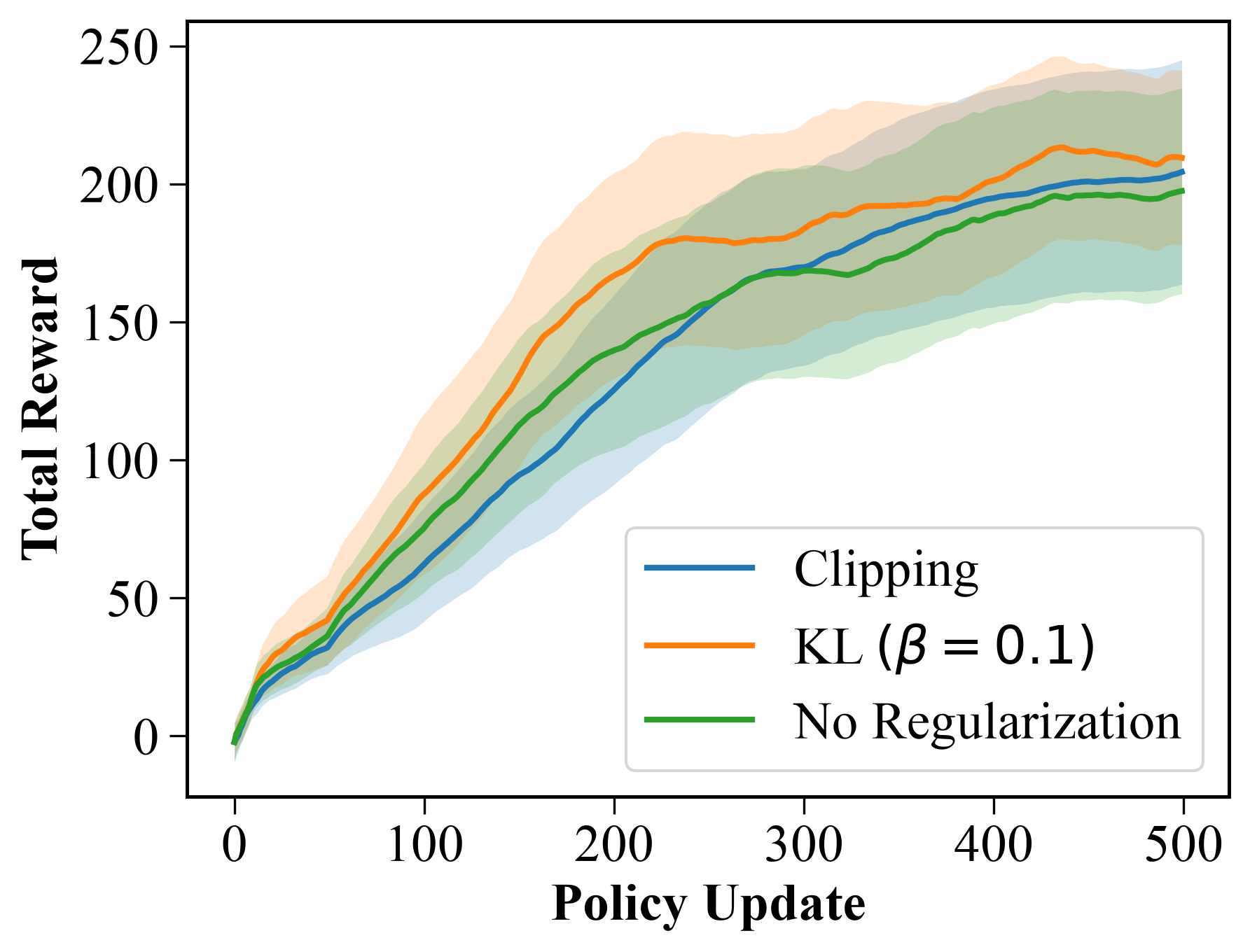} 
        \caption{Swimmer-v5}
    \end{subfigure}
    \hfill
    \caption{Ablation Study of Clipping versus KL Regularization. Shaded areas represent 95\% confidence intervals.}
    \label{fig:kl_study}
\end{figure}

\subsection{Alternate Choices of Second Policy in Loss Function}\label{sec:second_policy_choice}

In this section, we empirically compare the version of SP3O that uses $\pi_{\theta_{t-1}}$ as the second policy in the loss function to the version that uses $\pi_{\theta_\mathrm{ref}}$ as the second policy in the loss function. The training curves the two versions of the algorithm are shown in Figure~\ref{fig:two_is_ref_study}, with the time horizon set to 1000. The curves are smoothing with a sliding window of 50 updates. We observe that in Ant-v5 using $\pi_{\theta_{t-1}}$ as the second policy out performs setting the second policy to $\pi_{\theta_\mathrm{ref}}$. In the other environments, the two variants perform comparably.

\begin{figure}[htbp]
    \centering
    \begin{subfigure}[b]{0.32 \linewidth}
        \centering
        \includegraphics[width=\linewidth]{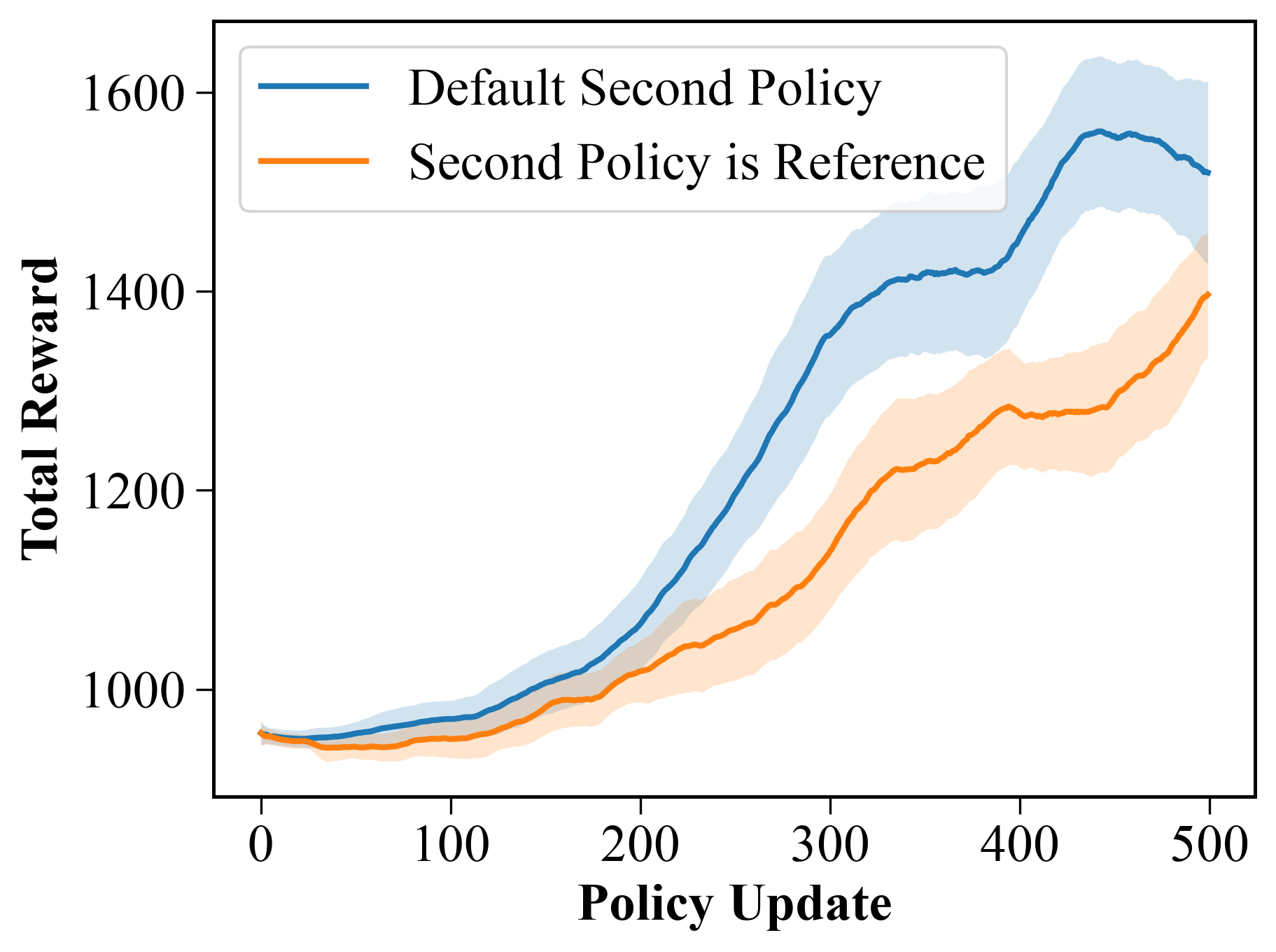} 
        \caption{Ant-v5}
    \end{subfigure}
    \hfill
    \begin{subfigure}[b]{0.32 \linewidth}
        \centering
        \includegraphics[width=\linewidth]{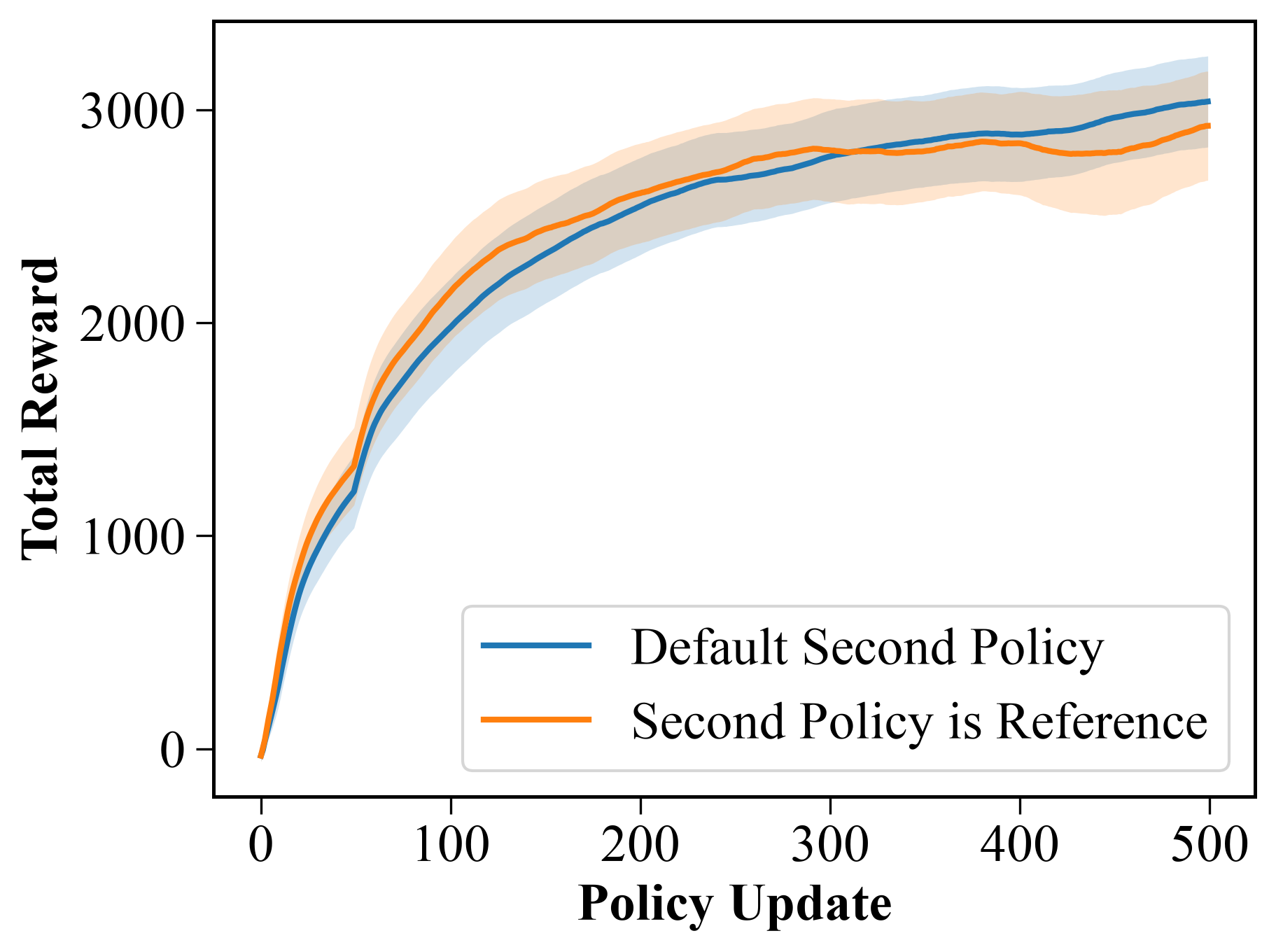} 
        \caption{HalfCheetah-v5}
    \end{subfigure}
    \hfill
    \begin{subfigure}[b]{0.32 \linewidth}
        \centering
        \includegraphics[width=\linewidth]{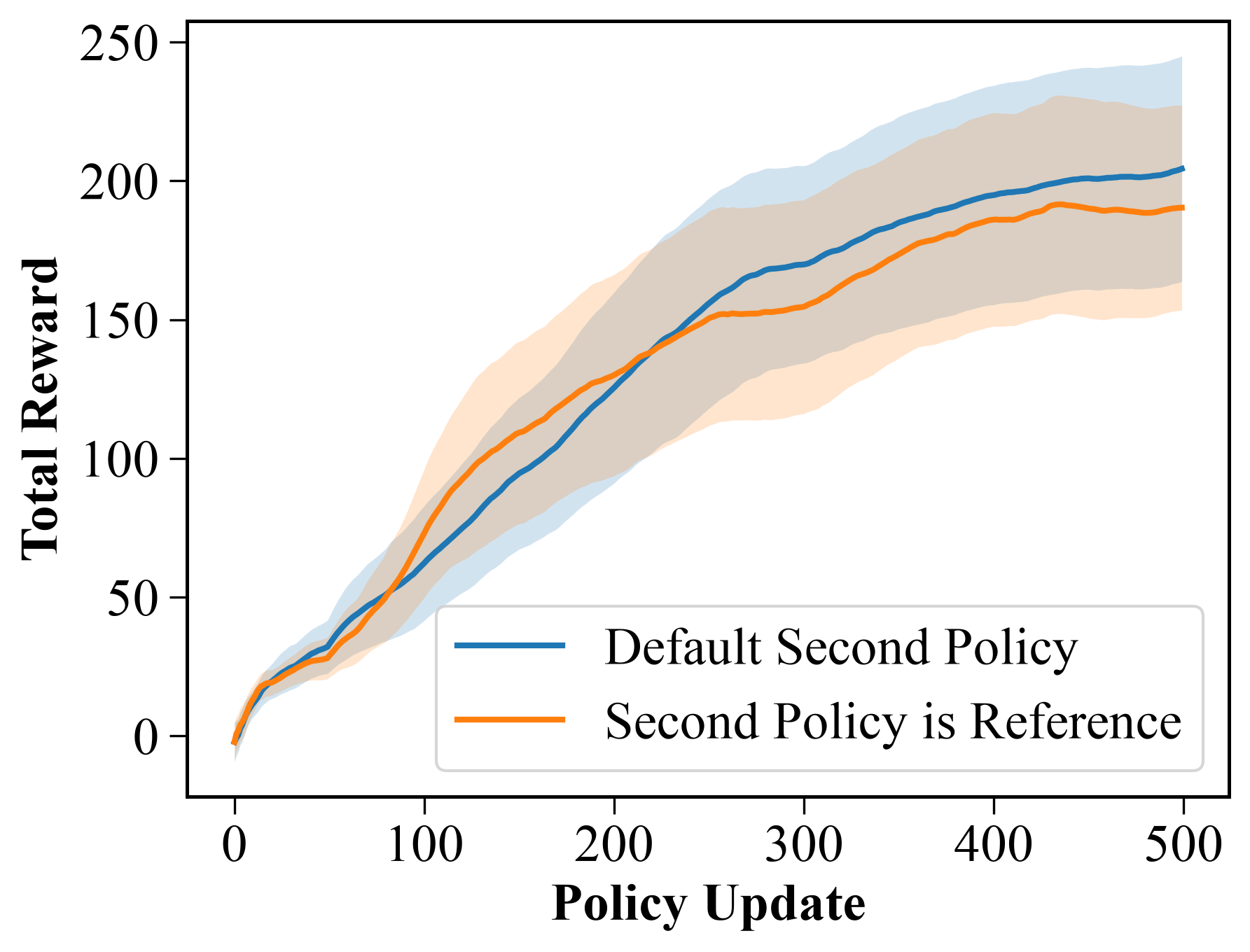} 
        \caption{Swimmer-v5}
    \end{subfigure}
    \hfill
    \caption{Ablation Study of Second Policy in Loss Function. Shaded areas represent 95\% confidence intervals.}
    \label{fig:two_is_ref_study}
\end{figure}

\section{Implementation Details}

All algorithms were run in Python 3.13.3 on Ubuntu 24.04.4. The robotic experiments were run on the CPU (a Ryzen Threadripper Pro with 565 GB of RAM). The LLM experiments were run on four NVIDIA RTX PRO 6000 Blackwell Max-Qs. Python libraries and their versions are included in the code appendix.

\subsection{Simulated Robotic Experiments.}\label{sec:implementation_details_rob}

All algorithms were implemented in PyTorch \citep{paszke2019pytorch}. All policy networks had two fully connected hidden layers with $64$ neurons each, and all activation functions were hyperbolic tangent (including the final layers). The Adam \citep{kingma2014adam} optimizer was used for all policy networks. During training, policies were Gaussian with a parameter controlling the log standard deviation for each dimension in the action space. However, during the evaluation, all policies were based on their deterministic version without the action noise. 

In each algorithm, before the total reward/return difference was plugged into the respective Bradley-Terry oracle, the difference was multiplied by an expertise factor. For ZPG, Online DPO, and P3O we use $10/H$ where $H$ is the time horizon. The expertise factor is chosen to be horizon dependent so that a unit difference in average reward between two trajectories will result in the same preference probability across all horizons. For SP3O, the expertise factor is $0.1$. This is not time dependent and higher than the algorithms because the SP3O preference model uses a discount factor, and with $\gamma = 0.99$, a unit difference in average reward will result in an increase of $\approx 1/(1-\gamma) = 100$ in the return difference. We chose $0.1$ since a unit increase in average reward results in the same preference probability when using segment of trajectory preferences.

\subsubsection{Global Hyperparameters}

All environments used $\gamma = 0.99$, except for Swimmer-v5, where we choose $\gamma = 0.999$ as suggested by \citet{franceschetti2022making}. According to the logic above, the expertise factor for SP3O in Swimmer-v5 is 0.01. For the main experiments we used seeds 42 through 141 inclusive, and the seeds for the L ablation experiment were 42 through 4041 inclusive.

\subsubsection{SP3O}

Although the theoretical basis for SP3O requires sampling states from the discounted state distribution, in practice, we sample states from each timestep uniformly. This is due to simplicity, and that we are evaluating policies based on total episode reward, not discounted reward.

Our estimate of the $Q$ function in the preference model, $\widehat{Q}$, used the actual discounted return of the rest of the trajectory at the end of the segment. However, if two segments $\sigma^1$ and $\sigma^2$ were $K_1$ and $K_2$ timesteps from the horizon $H$. Then, for both, we would only estimate $\widehat{Q}$ using the next $\min(K_1,K_2)$ timesteps.

For a segment pair $(\sigma^1,\sigma^2)$, the SP3O loss function will only take the gradient of the policy with respect to $\sigma^1$. Because of this asymmetry for every segment pair $(\sigma^1,\sigma^2)$ and estimated return difference $D(\sigma^1,\sigma^2)$, we can generate a new segment pair and return difference $(\sigma^2, \sigma^1), -D(\sigma^1,\sigma^2)$ without additional environmental interaction or queries to the preference oracle. Thus in practice for any segment pair we generate $(\sigma^1, \sigma^2)$ we will also use the segment pair $(\sigma^2, \sigma^1)$.

Similar to an advantage normalization for the training of PPO, the $D(\sigma^1_n,\sigma^2_n)$ in the preference dataset are normalized to standard deviation $1$ at each training step.

SP3O had all gradient norms clipped to $0.5$. Further hyperparameters are shown in Table~\ref{tab:sp3o_hyper}.

\begin{table*}[ht]
    \centering
        \begin{tabular}{lllllll}
        \toprule
             & Ant-v5 & HalfCheetah-v5 & Swimmer-v5 \\ 
        \midrule
            $\varepsilon$ & 0.2 & 0.2 & 0.2 \\
            Learning Rate & 3e-4 & 3e-4 & 3e-4 \\  
            Initial Log Std & -1.2 & -1.3 & 0.0  \\
            Epochs & 5 & 5 & 5 \\
        \bottomrule
        \end{tabular}
    \caption{SP3O Hyperparameters}
    \label{tab:sp3o_hyper}
\end{table*}

\subsubsection{P3O}

At each policy iteration, P3O sampled 10 trajectories from the environment. This gave a total of $45$ pairs per policy update. The preferences were based on a Bradley-Terry oracle using the difference in total reward sum. 

Similar to an advantage normalization for the training of PPO, the reward difference terms in the preference dataset are normalized to standard deviation $1$ at each training step.

P3O had all gradient norms clipped to $0.5$. Further hyperparameters are shown in Table~\ref{tab:p3o_hyper}.

\begin{table*}[ht]
    \centering
        \begin{tabular}{lllllll}
        \toprule
             & Ant-v5 & HalfCheetah-v5 & Swimmer-v5 \\ 
        \midrule
            $\varepsilon$ & 0.2 & 0.2 & 0.2 \\
            Initial Log Std & -1.2 & -1.3 & 0.0  \\
            Learning Rate & 3e-4 & 3e-4 & 3e-4 \\  
            Epochs & 5 & 5 & 5 \\
        \bottomrule
        \end{tabular}
    \caption{P3O Hyperparameters}
    \label{tab:p3o_hyper}
\end{table*}

\subsubsection{Online DPO}

At each policy iteration, Online DPO sampled 10 trajectories from the environment. This gave a total of $45$ pairs per policy update. The preferences were based on a Bradley-Terry oracle using the difference in total reward sum. Online DPO had all gradient norms clipped to $0.5$. Further hyperparameters are shown in Table~\ref{tab:online_dpo_hyper}.

\begin{table*}[ht]
    \centering
        \begin{tabular}{lllllll}
        \toprule
             & Ant-v5 & HalfCheetah-v5 & Swimmer-v5  \\ 
        \midrule
            $\beta$ & 0.1 & 0.1 & 0.1 \\
            Initial Log Std & -1.2 & -1.3 & 0.0  \\
            Learning Rate & 7e-4 & 7e-4 & 7e-4 \\
            Epochs & 5 & 5 & 5 \\
        \bottomrule
        \end{tabular}
    \caption{Online DPO Hyperparameters}
    \label{tab:online_dpo_hyper}
\end{table*}

\subsubsection{ZPG}

 At each policy iteration, ZPG sampled 10 trajectories from the environment. This gave a total of $25$ pairs per policy update. This value is different from that of Online DPO and P3O, since half of ZPG's trajectories are generated using the perturbed policy. The learning rate was $0.003$, the perturbation distance was $0.1$, and there was one epoch per policy update. Our implementation of ZPG also used an Adam \citep{kingma2014adam} style momentum system at each update step.

\subsection{LLM Finetuning Experiments.}\label{sec:implementation_details_llm}

\subsubsection{Hyperparameters}

LoRA adapters were applied to all linear layers with rank $16$, scaling $\alpha=32$, and dropout $0.05$. 

The AdamW \citep{loshchilov2017decoupled} optimizer was used for all algorithms with $\beta_1=0.9$, $\beta_2=0.999$, and $\epsilon=10^{-8}$. The global gradient norm is clipped to $1.0$. 

Shown in Figure~\ref{fig:lr_tuning} are the toxicity versus KL curves of each algorithm with different learning rates. For the regularization parameters in these experiments we use $\beta = 1.0$ for SP3O-KL, $\varepsilon = 0.2$ for SP3O-Clip, $\varepsilon = 0.2$ for P3O, and $\beta = 0.5$ for DPO. 

In P3O, SP3O-KL, and SP3O-Clip, we observe that the curve does not improve when the learning rate goes lower than 1e-5. In DPO, we observe that the curve does not improve when the learning rate goes lower than 2e-5. These are the learning rate values we use for the main experiments.

\begin{figure}[htbp]
    \centering
    \begin{subfigure}[b]{0.24 \linewidth}
        \centering
        \includegraphics[width=\linewidth]{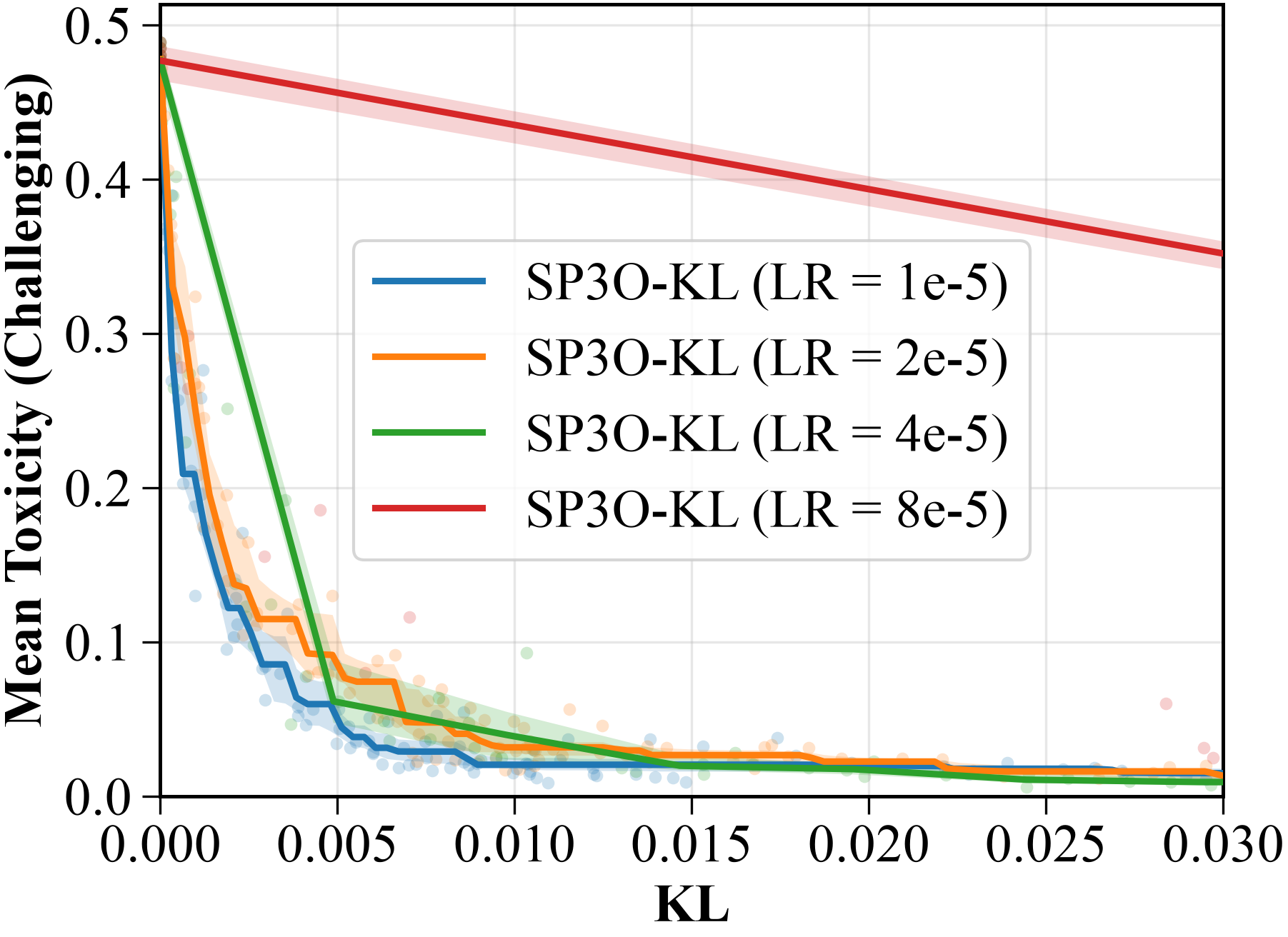} 
    \end{subfigure}
    \hfill
    \begin{subfigure}[b]{0.24 \linewidth}
        \centering
        \includegraphics[width=\linewidth]{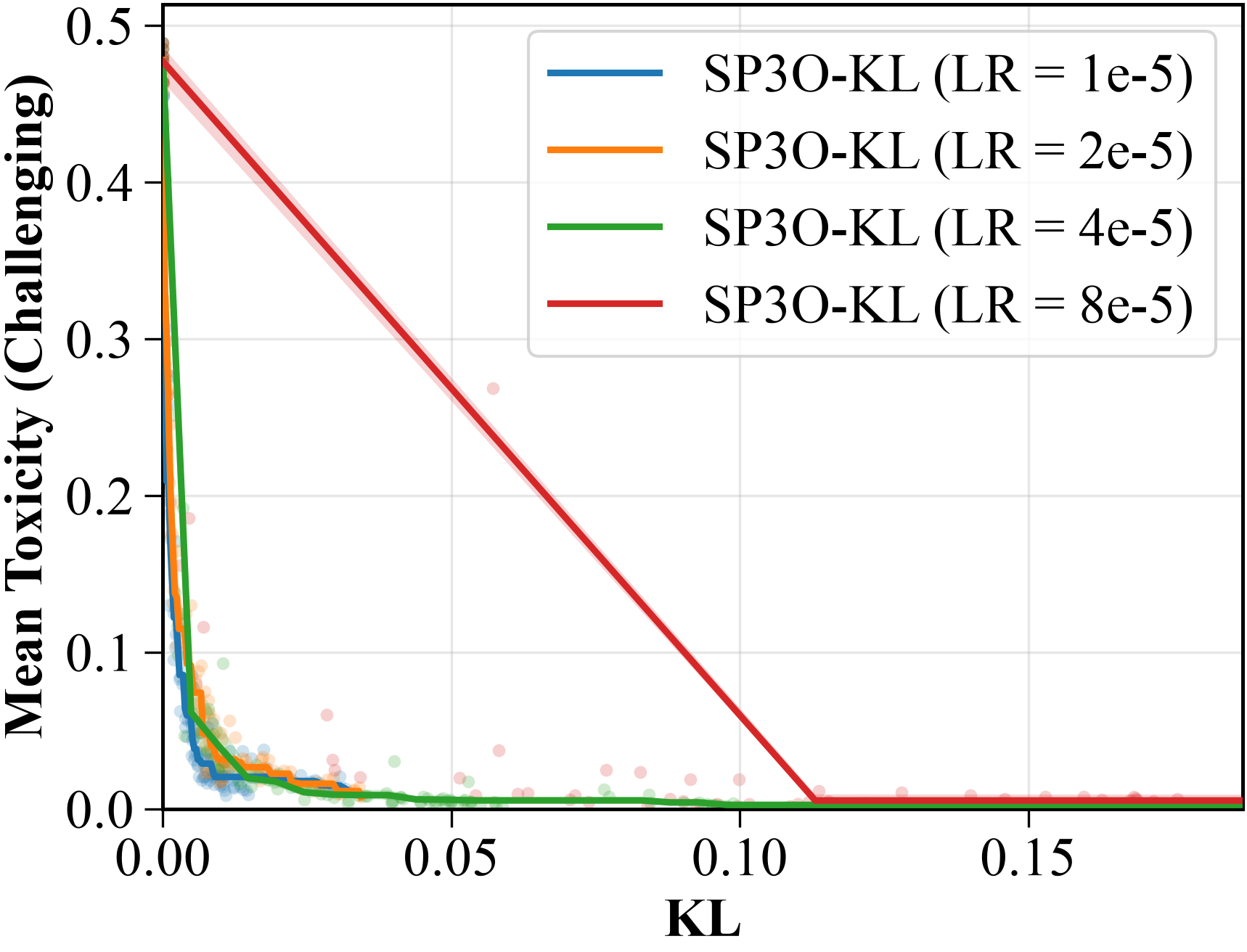} 
    \end{subfigure}
    \hfill
    \begin{subfigure}[b]{0.24 \linewidth}
        \centering
        \includegraphics[width=\linewidth]{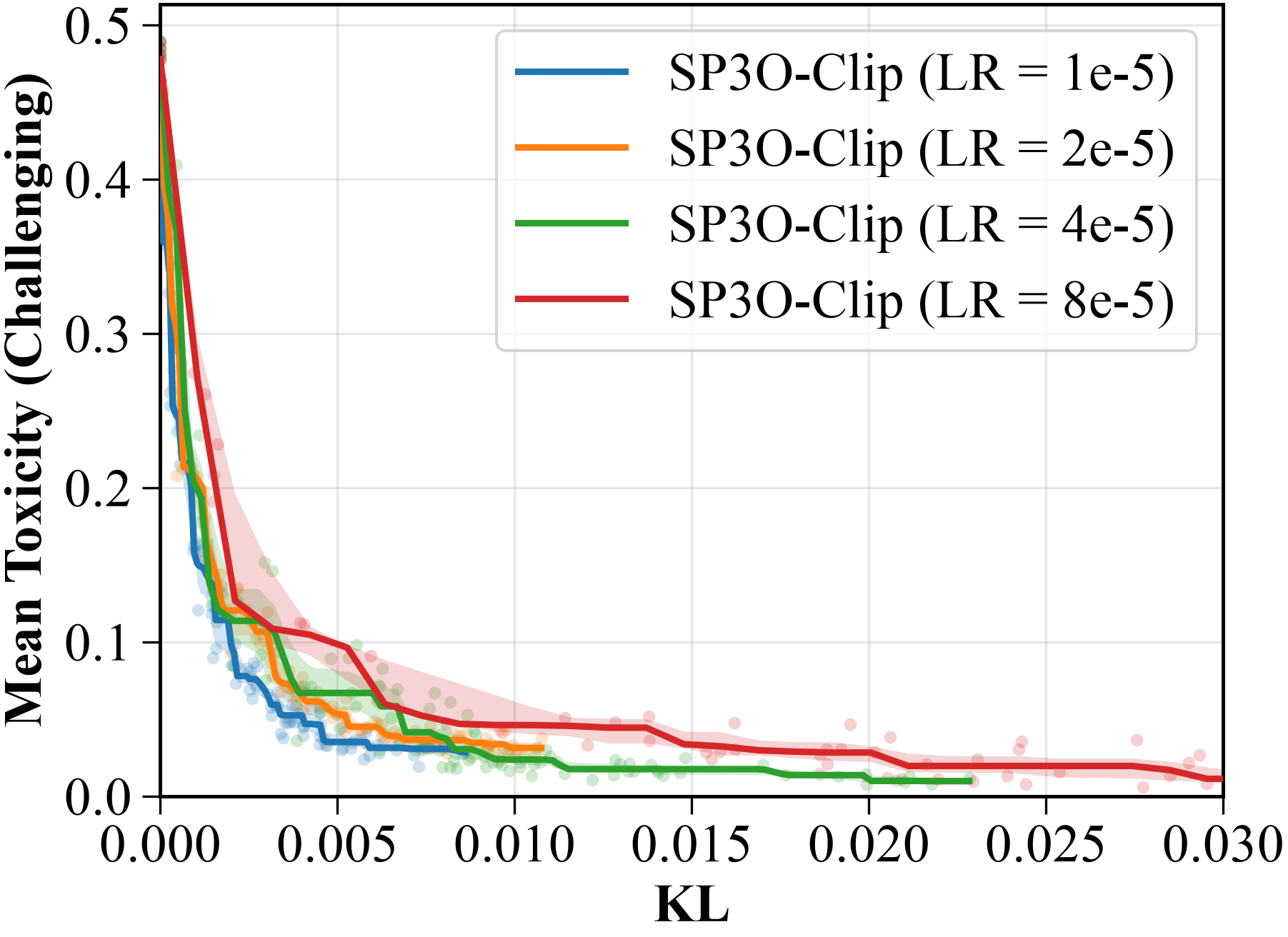} 
    \end{subfigure}
    \hfill
    \begin{subfigure}[b]{0.24 \linewidth}
        \centering
        \includegraphics[width=\linewidth]{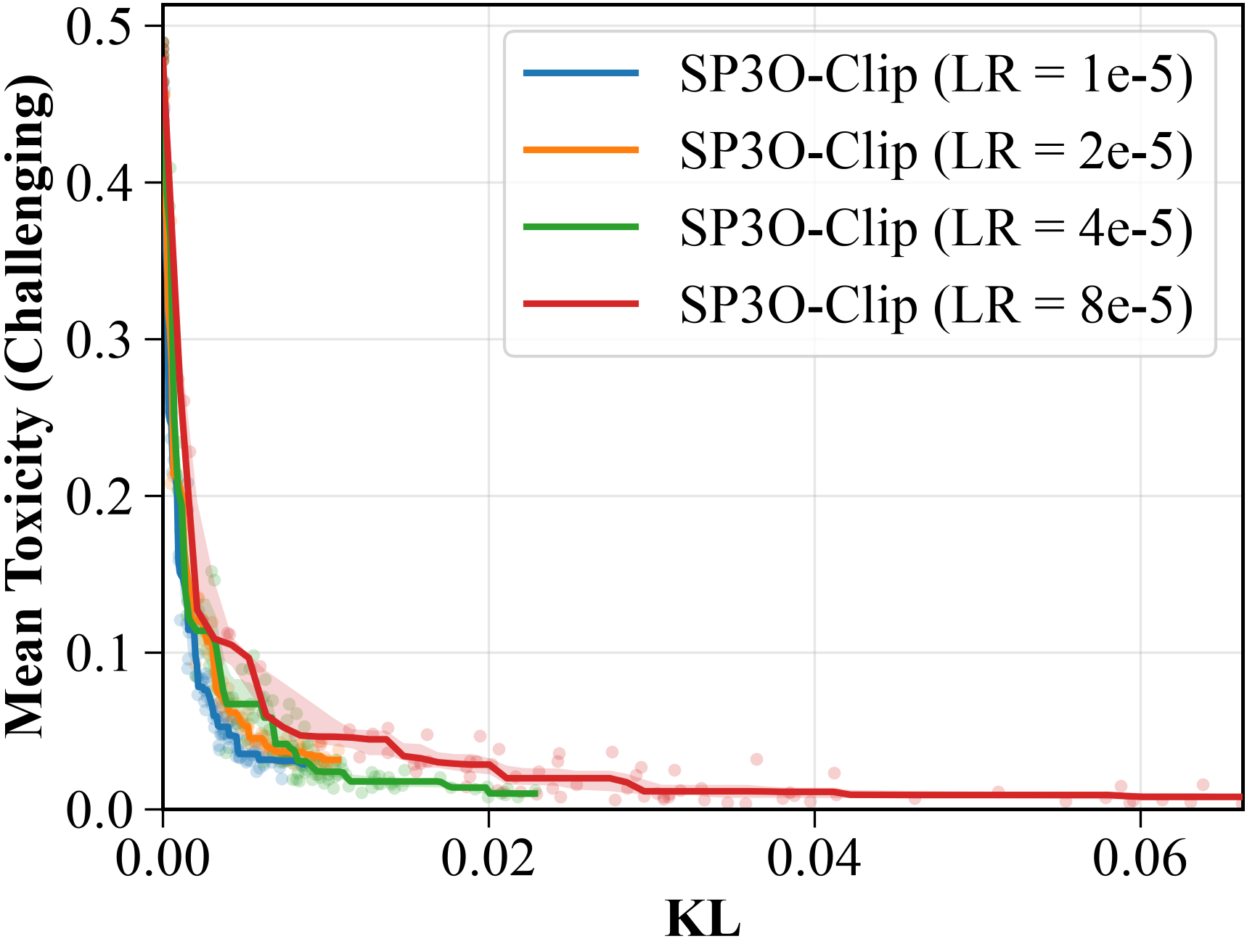} 
    \end{subfigure}
    \hfill
    \begin{subfigure}[b]{0.24 \linewidth}
        \centering
        \includegraphics[width=\linewidth]{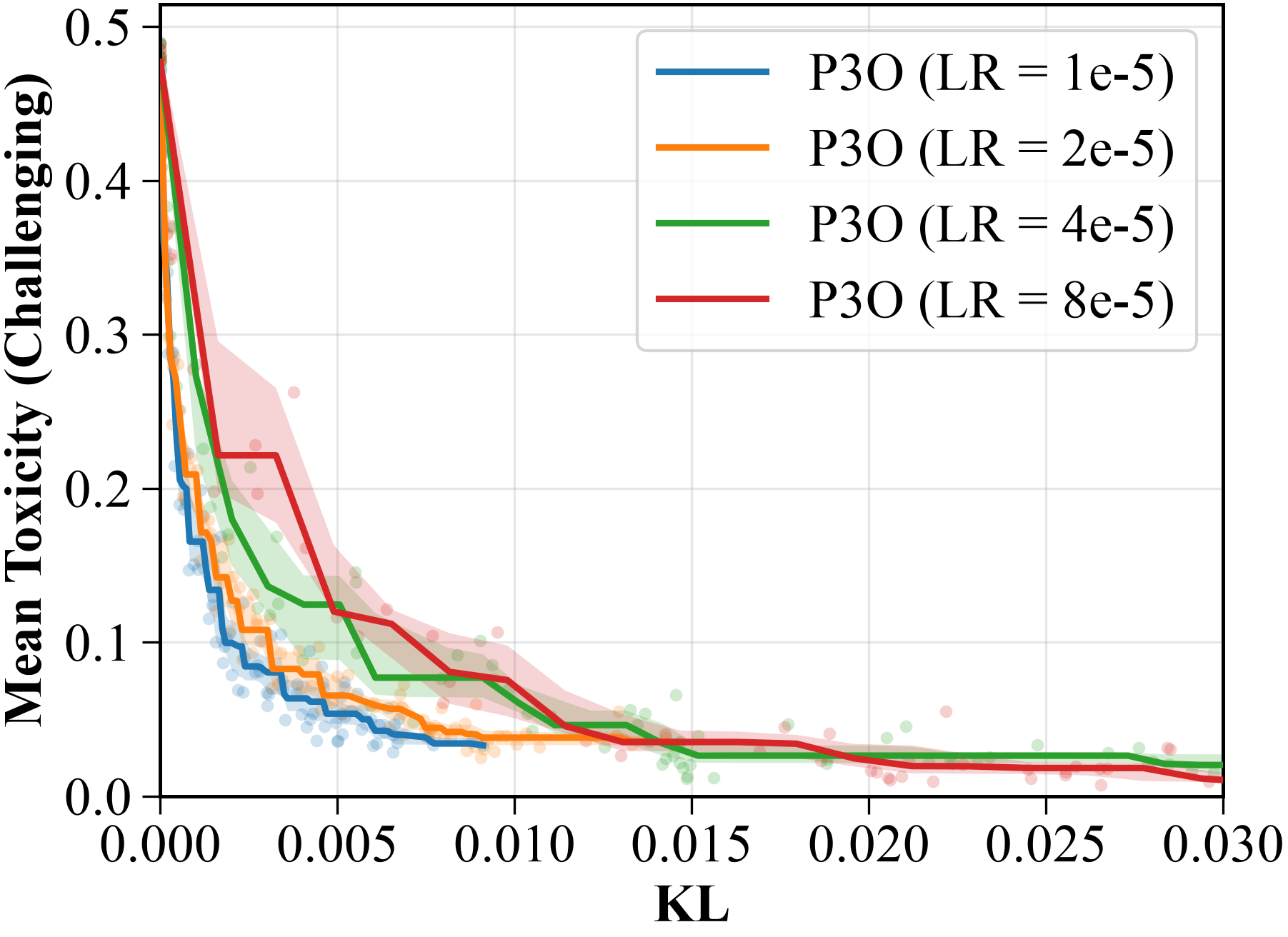} 
    \end{subfigure}
    \hfill
    \begin{subfigure}[b]{0.24 \linewidth}
        \centering
        \includegraphics[width=\linewidth]{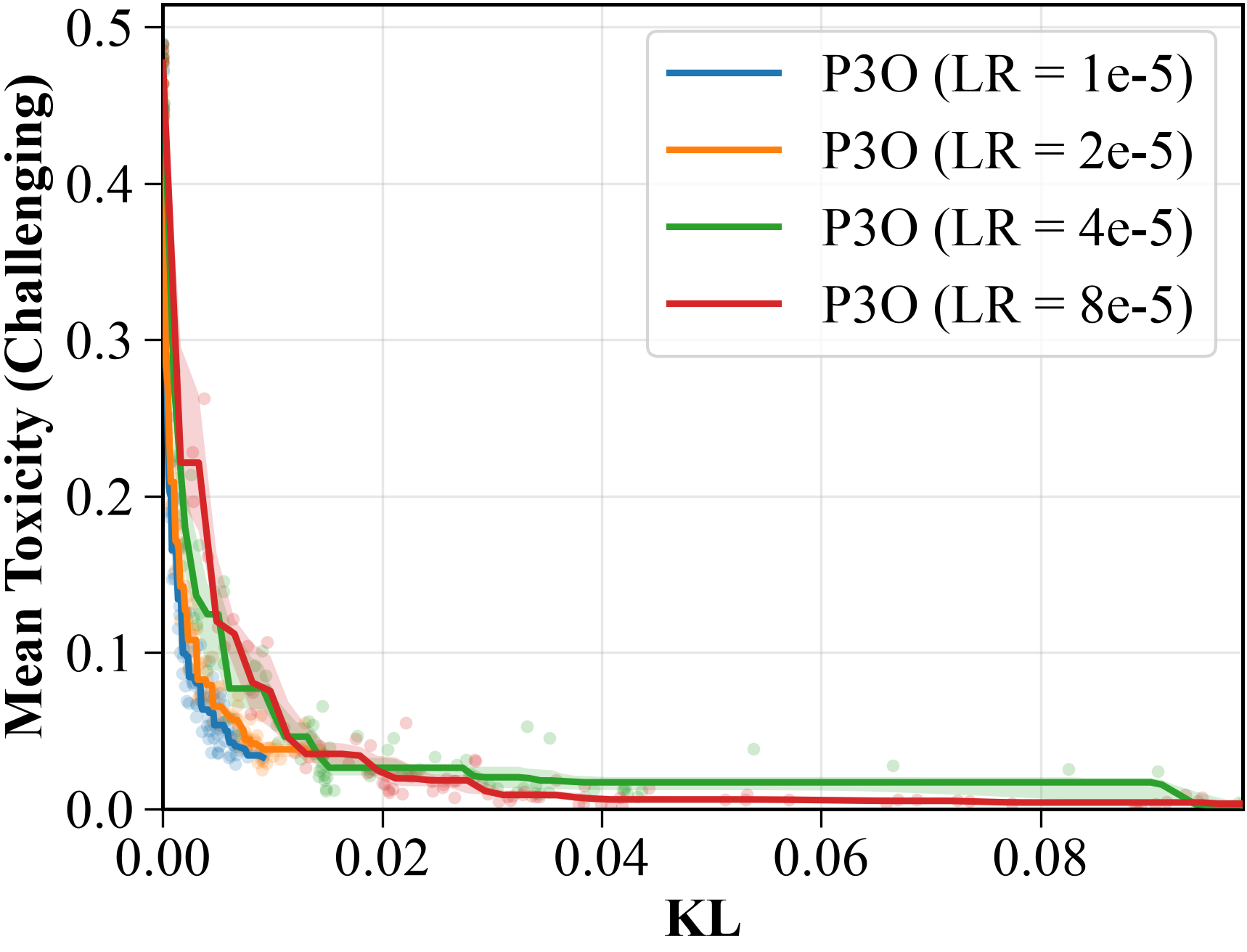} 
    \end{subfigure}
    \hfill
    \begin{subfigure}[b]{0.24 \linewidth}
        \centering
        \includegraphics[width=\linewidth]{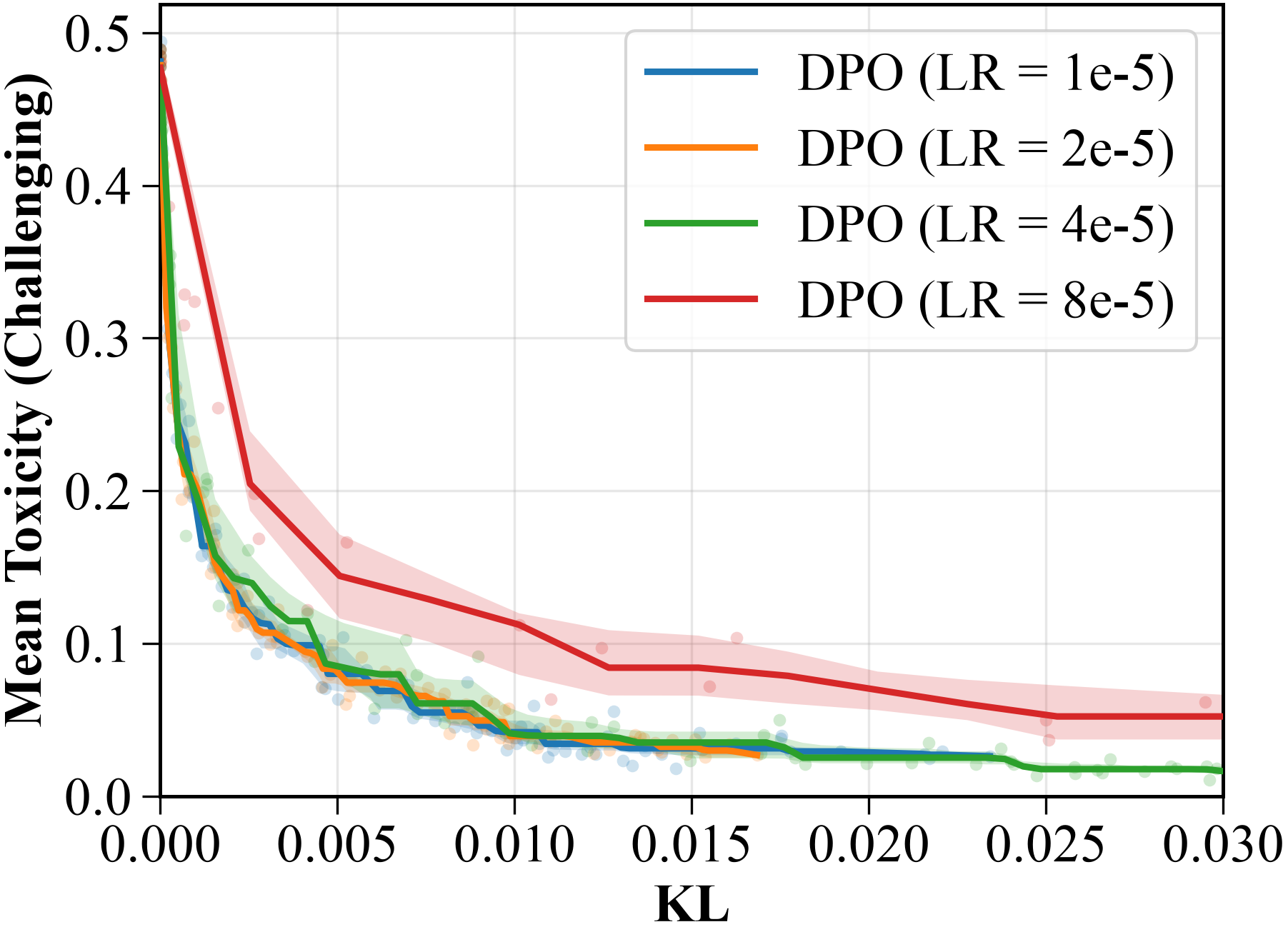} 
    \end{subfigure}
    \hfill
    \begin{subfigure}[b]{0.24 \linewidth}
        \centering
        \includegraphics[width=\linewidth]{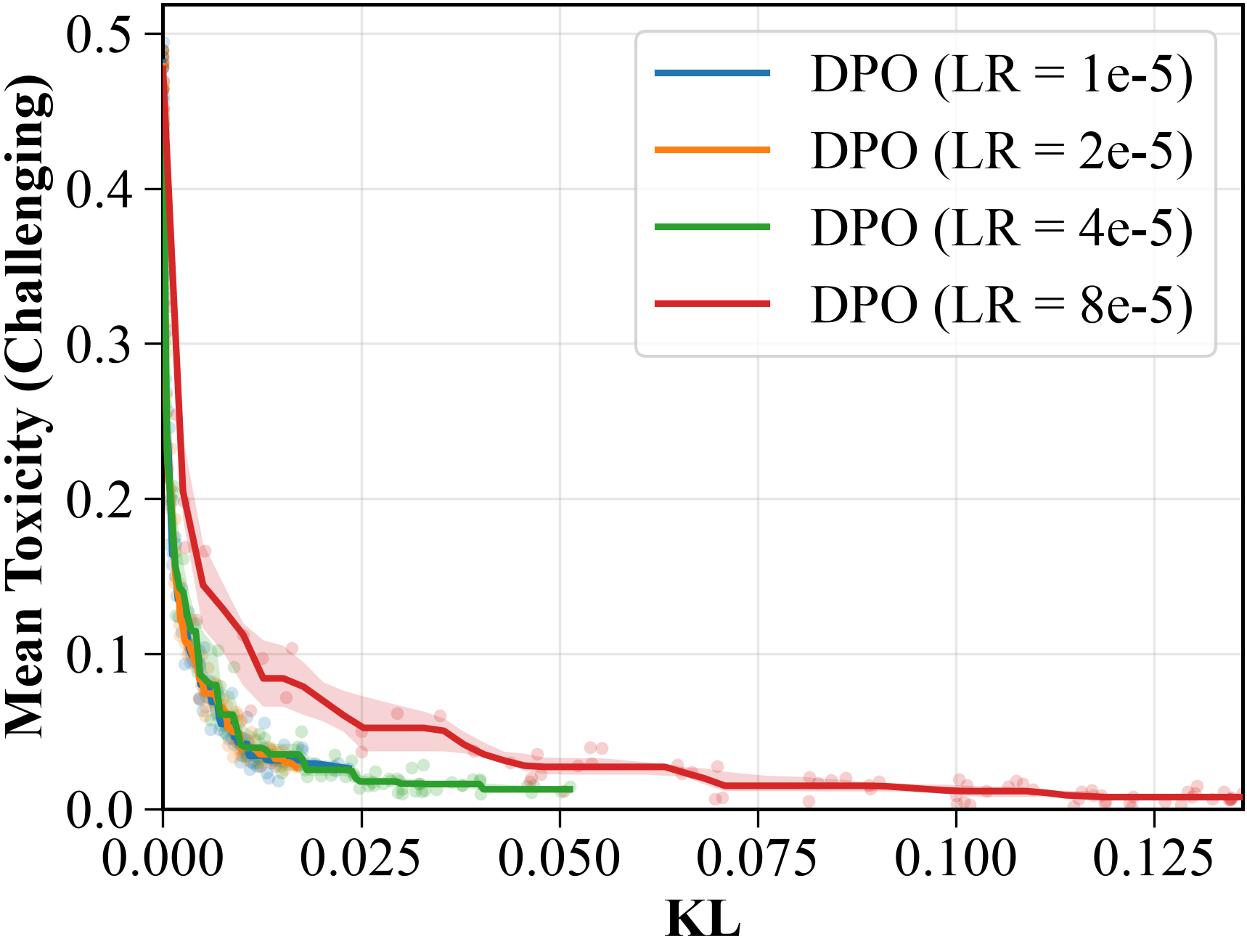} 
    \end{subfigure}
    \hfill
    \caption{Learning Rate Tuning Results.}
    \label{fig:lr_tuning}
\end{figure}

The KL divergence is calculated by randomly picking 1000 prompts from the entire dataset. For all KL divergences we use the $k_3$ estimator \citep{schulman2020approximating}. All log probabilities of sentences and full continuations are calculated with a mean reduction to avoid issues with varying lengths. The seeds for all experiments were $42,43,44,45,46$. 

In SP3O and P3O, we set the reward difference terms to $+1$ or $-1$ depending on which segment/continuation is preferred, and don't infer any magnitude of the preference. In P3O, we use the P3O-V1 (Clipping Separately) from the original P3O paper \citep{wu2023pairwiseproximalpolicyoptimization}.

In all algorithms, we did not update the reference policy in the loss functions from the base model, despite sampling from the current policy. This is because the difference between the models is small, and keeping the base model as the reference gives better regularization. In Figure~\ref{fig:ref_policy}, we provide a comparison on updating the reference policy to the current generating policy versus keeping the base policy as the reference in the loss function. We observe that for all algorithms keeping the reference policy as the base model provides a better curve.

\begin{figure}[htbp]
    \centering
    \begin{subfigure}[b]{0.45 \linewidth}
        \centering
        \includegraphics[width=\linewidth]{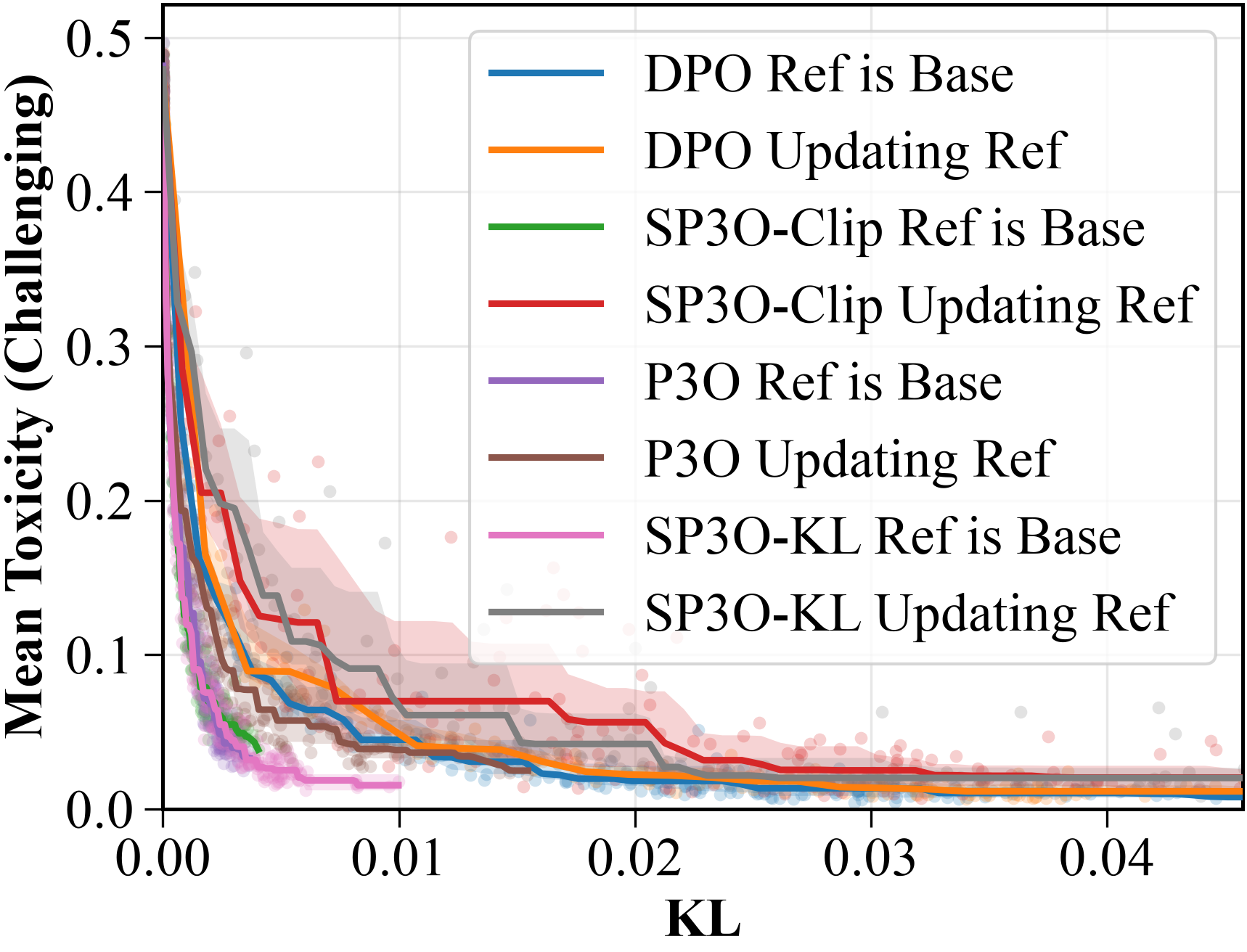} 
    \end{subfigure}
    \hfill
    \begin{subfigure}[b]{0.49 \linewidth}
        \centering
        \includegraphics[width=\linewidth]{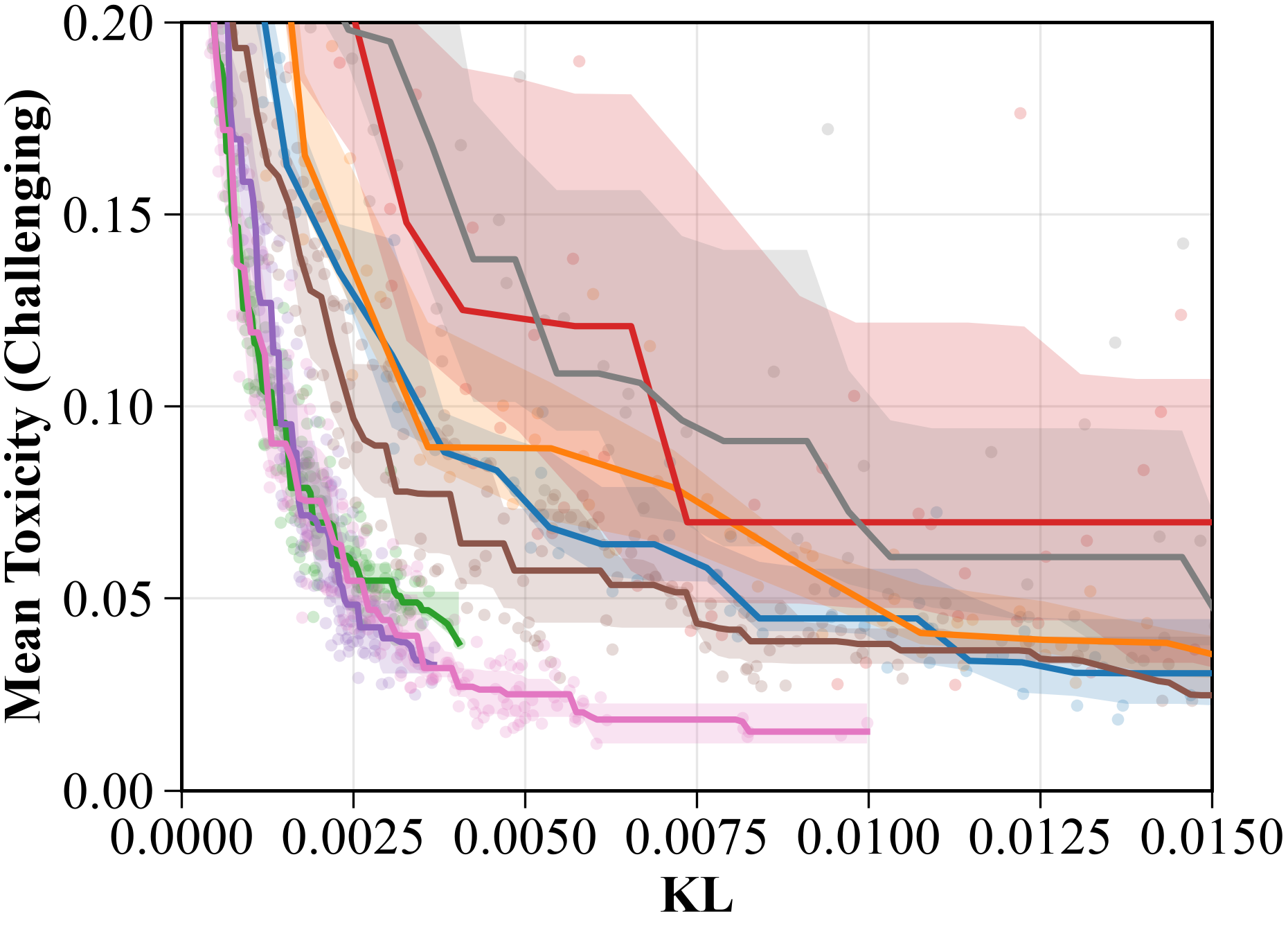} 
    \end{subfigure}
    \caption{Comparison of Changing Reference Policy.}
    \label{fig:ref_policy}
\end{figure}

\subsubsection{Fitting Toxicity vs KL Curves}

We fit these curves with the Pool-Adjacent-Violators Algorithm (PAVA), using the cloud of points generated my each training runs checkpoints. For the 95\% confidence intervals are found through bootstrapping 300 times.

\subsubsection{SP3O Sentence Splitting}

We split a continuation at points where a period, question mark, or exclamation mark appears and at the end of the continuation. We note that these punctuation marks can show up within a sentence such as in the title "Dr.". This creates a small disadvantage for SP3O, as the created sentences could be very short. However, this kind of situation is rare enough such that the impact on SP3O's performance should be negligible. 

\section{Theoretical Analysis and Proof of Theorems}

\subsection{Proof of Lemma~\ref{lemma:pdsurrogate}}\label{sec:obj_diff_proof}

\textbf{Lemma 1} \begin{it}
Suppose we have three policies $\pi_{1}$, $\pi_2$ and $\pi_{\mathrm{ref}}$. We assume that the support of $\pi_1$ and $\pi_2$ is enclosed within the support of $\pi_\mathrm{ref}$, that is $\pi_1(s|a) > 0$ or $\pi_2(s|a) > 0$ implies $\pi_\mathrm{ref}(s|a) > 0, \forall s \in \mathcal{S}, a \in \mathcal{A}$. Then we can estimate the objective difference of $\pi_1$ and $\pi_2$ using only $\pi_\mathrm{ref}$ as follows.
\end{it}

\begin{proof}
Firstly, we can estimate the objective using only the first $H$ timesteps due to the discount factor. Specifically, the error between two trajectories is:
\begin{align*}
    &\left | \sum_{t=1}^\infty \gamma^{t-1}r(s_t,a_t) - \sum_{t=1}^H \gamma^{t-1}r(s_t,a_t) \right |
    =  \gamma^{H}\sum_{t=1}^\infty \gamma^{t-1}r(s_t,a_t)
    \leq  \gamma^H \sum_{t=1}^\infty \gamma^{t-1}
    = \gamma^H \frac{1}{1-\gamma}.
\end{align*}
Given this the error in estimating the expected value difference surrogate will be at most $\frac{2\gamma^H}{1-\gamma}$, which can be very small given $H$ is the effective horizon. Then we can estimate the difference in expected value as:
\begin{align*}
    &J_\mathcal{M}(\pi_1) - J_\mathcal{M}(\pi_2) \\
    =& \E_{a^1_t \sim \pi_1(\cdot | s^1_t), s^1_1 \sim d_0}\left[ \sum_{t=1}^H \gamma^{t-1}r(s^1_t,a^1_t) \right] - \E_{a^2_t \sim \pi_2(\cdot \mid s^2_t), s^2_1 \sim d_0}\left[ \sum_{t=1}^H \gamma^{t-1}r(s^2_t,a^2_t) \right]+\mathcal{O}(\gamma^H)\\
    =& \E_{\substack{a^1_t \sim \pi_1(\cdot | s_t), s^1_1 \sim d_0, \\ a^2_t \sim \pi_2(\cdot | s^2_t),s^2_1 \sim d_0}}\Bigg[ \sum_{t=1}^H \gamma^{t-1}r(s^1_t,a^1_t) - \sum_{t=1}^H \gamma^{t-1}r(s^2_t,a^2_t) \Bigg]
    +\mathcal{O}(\gamma^H)\\
    =& \E_{\substack{a^1_t \sim \pi_\mathrm{ref}(\cdot | s^1_t), s^1_1 \sim d_0, \\ a^2_t \sim \pi_\mathrm{ref}(\cdot | s^2_t),s^2_1 \sim d_0}}\Bigg [  \frac{\prob_{\pi_1}(\tau^1)\prob_{\pi_2}(\tau^2)}{\prob_{\pi_\mathrm{ref}}(\tau^1)\prob_{\pi_\mathrm{ref}}(\tau^2)} \left[ \sum_{t=1}^H \gamma^{t-1}r(s_t,a_t) - \sum_{t=1}^H \gamma^{t-1}r(s'_t,a'_t) \right] \Bigg]+\mathcal{O}(\gamma^H),
\end{align*} where $\tau^1=\{(s_t^1, a_t^1)\}_{t=1}^H$ and $\tau^2=\{s^2_t, a_t^2\}_{t=1}^H.$

Now, let us analyze each importance sampling term as follows: 
\begin{align*}
    \frac{\prob_{\pi_1}(\tau^1)}{\prob_{\pi_\mathrm{ref}}(\tau^1)} 
    =&  \frac{d_0(s^1_1)\pi_{1}(a^1_1|s^1_1) P(s^1_2|s^1_1, a^1_1) \cdots P(s_H^1|s_{H-1}^1,a_{H-1}^1)\pi_1(a_H^1|s_H^1)}{d_0(s^1_1)\pi_{\mathrm{ref}}(a^1_1|s^1_1) P(s^1_2|s^1_1, a^1_1) \cdots P(s_H^1|s_{H-1}^1,a_{H-1}^1)\pi_\mathrm{ref}(a_H^1|s_H^1)} \\
    =& \prod_{h=1}^H \frac{\pi_1(a^1_h|s^1_h)}{\pi_{\mathrm{ref}}(a^1_h|s^1_h)},
\end{align*}
where the last equality is due to the same transition in each trajectory. Therefore, we can express the expected value difference as:
\begin{align*}
    &J_\mathcal{M}(\pi_1) - J_\mathcal{M}(\pi_2)\\
    =& \E_{\substack{a^1_t \sim \pi_\mathrm{ref}(\cdot | s^1_t), s^1_1 \sim d_0, \\ a_t^2 \sim \pi_\mathrm{ref}(\cdot | s^2_t),s^2_1 \sim d_0}}\Bigg [ \prod_{t=1}^H\frac{\pi_1(a_t|s_t)}{\pi_\mathrm{ref}(a_t|s_t)} \frac{\pi_2(a_t|s_t)}{\pi_\mathrm{ref}(a_t|s_t)} \left[ \sum_{t=1}^H \gamma^{t-1}r(s^1_t,a^1_t) - \sum_{t=1}^H \gamma^{t-1}r(s^2_t,a^2_t) \right]\Bigg]+\mathcal{O}\left(\gamma^H\right).
\end{align*}
\end{proof}

\subsection{Proof of Lemma \ref{lemma:q}} \label{sec:qlemmaproof}

\begin{lemma}\label{lemma:q}
     For all $s\in \mathcal{S}$ and $a\in \mathcal{A}$, we have $Q_{\mathcal{M'}}^{\pi_\mathrm{ref}}(s,a) = Q_{\mathcal{M}}^{\pi_\mathrm{ref}}(s,a)$.
\end{lemma}
\begin{proof}
Given some $s_1 \in \mathcal{S}, a_1 \in \mathcal{A}$. If $s_1$ is one step from termination then it follows from definitions that $Q^{\pi_\mathrm{ref}}_{\mathcal{M'}}(s_1,a_1) = r'(s_1,a_1) = Q^{\pi_\mathrm{ref}}_{\mathcal{M}}(s_1,a_1)$. In the other case, assume without loss of generality that $s_1$ is $k > 1$ steps from termination in $\mathcal{M'}$. Then:
\begin{align*}
    Q_{\mathcal{M}'}^{\pi_\mathrm{ref}}(s_1,a_1) 
    =& \E_{\pi_\mathrm{ref}} \left [\sum_{t=1}^{k} \gamma^{t-1}r'(s_t,a_t) \right]\\
    =& \sum_{t=1}^{k}\gamma^{t-1}\E_{\pi_\mathrm{ref}}[r'(s_t,a_t) ] \\
    =& \gamma^{k-1}\E_{\pi_\mathrm{ref}}[ Q^{\pi_{\mathrm{ref}}}_\mathcal{M}(s_{k},a_{k}) ]+ \sum_{t=1}^{k-1}\gamma^{t-1}\E_{\pi_\mathrm{ref}}[r(s_t,a_t) ] \\ 
    =& Q_\mathcal{M}^{\pi_\mathrm{ref}}(s_1,a_1).
\end{align*}
which concludes the proof.$\hfill\blacksquare$
\end{proof}

\subsection{Proof of Lemma \ref{lemma:occupancy}} \label{sec:lemma2proof}

\begin{lemma}\label{lemma:occupancy}
    For any $s\in \mathcal{S}$, we have $d_{\mathcal{M'}}^{\pi_\mathrm{ref}}(s) = d_{\mathcal{M}}^{\pi_\mathrm{ref}}(s)$.
\end{lemma}
\begin{proof}
Given some $s \in \mathcal{S}$, by the definition of the discounted occupancy defined, we have:
\begin{align*}
    d_{\mathcal{M'}}^{\pi_\mathrm{ref}}(s) 
    =& \frac{1-\gamma}{1-\gamma^L} \sum_{l=1}^L \gamma^{l-1}  \prob[s_{l} = s \mid \pi_\mathrm{ref}, s_1 \sim d_0']  \\
    =& \frac{1-\gamma}{1-\gamma^L} \sum_{l=1}^L \sum_{t=1,L+1,2L+1,\dots} \Big (\gamma^{t-1} (1-\gamma^L) \gamma^{l-1}  \prob[s_{l+t-1} = s \mid \pi_\mathrm{ref}, s_t \sim d_{\mathcal{M},t}^{\pi_\mathrm{ref}}] \Big )  \\
    =& (1-\gamma)\sum_{t=1,L+1,2L+1,\dots} \sum_{l=1}^L \gamma^{t+l-2} 
    \sum_{s' \in \mathcal{S}} \prob[s_{l+t-1} = s \mid \pi_\mathrm{ref}, s_t = s']\prob[s_t = s' \mid \pi_\mathrm{ref}, s_1 \sim d_0]  \\
    =& (1-\gamma)\sum_{t=1,L+1,2L+1,\dots} \sum_{l=1}^L \Big ( \gamma^{t+l-2} \prob[s_{l+t-1} = s \mid \pi_\mathrm{ref},  s_1 \sim d_0] \Big ) \\
    =& (1-\gamma) \sum_{t=1}^\infty \gamma^{t-1} \prob[s_t = s \mid \pi_\mathrm{ref} , s_1 \sim d_0] \\
    =& d_\mathcal{M}^{\pi_\mathrm{ref}}(s).
\end{align*}
which concludes the proof.$\hfill\blacksquare$
\end{proof}

\subsection{Proof of Theorem \ref{thm:difference_theorem_statement}}\label{sec:theorem1proof}

\textbf{Theorem 1} \begin{it}
If we choose $\pi_{\mathrm{ref}} = \pi_\theta$, the gradient of the expected policy value in $\mathcal{M}$ is a scaled version of the gradient of the expected policy value difference in $\mathcal{M'}$, i.e., for an arbitrary policy $\pi'$, we have
\begin{align*}
    \nabla_\theta J_\mathcal{M}(\pi_\theta) = \frac{1}{1-\gamma^L} \nabla_{\theta}(J_{\mathcal{M'}}(\pi_\theta) - J_{\mathcal{M'}}(\pi')).
\end{align*}
\end{it}
\begin{proof}
    
Firstly the $\nabla_{\theta}(J_{\mathcal{M'}}(\pi_\theta) - J_{\mathcal{M'}}(\pi'))$ simplifies to $\nabla_\theta J_\mathcal{M'}(\pi_\theta)$. Now we will expand $\nabla_\theta J_\mathcal{M'}(\pi_\theta)$ in a similar style to the Policy Gradient Theorem~\citep{sutton1998reinforcement}. 
\begin{align*}
\nabla_\theta J_{\mathcal{M'}}(\pi_\theta) 
=& \nabla_{\theta}\E_{s_1 \sim  d_0'}[V_{\mathcal{M'}}^{\pi_\theta}(s_1)] \\
=& \sum_{s_1 \in \mathcal{S}} \nabla_\theta(d_0'(s_1)V_{\mathcal{M'}}^{\pi_\theta}(s_1)) \\
=& \sum_{s_1 \in \mathcal{S}} (\nabla_\theta d_0'(s_1))V_{\mathcal{M'}}^{\pi_\theta}(s_1) + d_0'(s_1)\nabla_\theta V_{\mathcal{M'}}^{\pi_\theta}(s_1)).
\end{align*}
Since $d_0'$ doesn't depend on $\theta$ we find that this equals:
\begin{align*}
\nabla_\theta J_{\mathcal{M'}}(\pi_\theta) =& \sum_{s_1 \in \mathcal{S}} d_0'(s_1)\nabla_\theta V_{\mathcal{M'}}^{\pi_\theta}(s_1)]  \\
=& \E_{s_1 \sim d_0'}\nabla_\theta V_{\mathcal{M'}}^{\pi_\theta}(s_1)]  \\
=&  \E_{s_1 \sim d_0'} \left[ \nabla_\theta \sum_{a_1 \in \mathcal{A}} \pi_\theta(a_1 | s_1) Q_{\mathcal{M'}}^{\pi_\theta}(s_1,a_1)\right] \\
=&  \E_{s_1 \sim d_0',a_1\sim \pi_\theta(\cdot | s_1)} \left[ \nabla_\theta  \ln\pi_\theta (a_1|s_1) Q_{\mathcal{M'}}^{\pi_\theta}(s_1,a_1)\right] +  \E_{s_1 \sim d_0',a_1\sim \pi_\theta(\cdot | s_1)} \left[ \nabla_\theta  Q_{\mathcal{M'}}^{\pi_\theta}(s_1,a_1) \right] \\
=&  \E_{s_1 \sim d_0',a_1\sim \pi_\theta(\cdot | s_1)} \left[ \nabla_\theta  \ln\pi_\theta (a_1|s_1) Q_{\mathcal{M'}}^{\pi_\theta}(s_1,a_1)\right] +  \E_{s_1 \sim d_0',a_1\sim \pi_\theta(\cdot | s_1)} \left[ \nabla_\theta  r'(s_1,a_1) \right] \\ &+ \E_{s_1 \sim d_0',a_1\sim \pi_\theta(\cdot | s_1)} \left[ \nabla_\theta \gamma V_{\mathcal{M'}}^{\pi_\theta}(s_2) \right].
\end{align*}
Since when $k < L$, $r_k'(s,a) = r(s,a)$ we get:
\begin{align*}
\nabla_\theta J_{\mathcal{M'}}(\pi_\theta)
=&  \E_{s_1 \sim d_0',a_1\sim \pi_\theta(\cdot | s_1)} \left[ \nabla_\theta  \ln\pi_\theta (a_1|s_1) Q_{\mathcal{M'}}^{\pi_\theta}(s_1,a_1)\right] +  \E_{s_1 \sim d_0',a_1\sim \pi_\theta(\cdot | s_1)} \left[ \nabla_\theta  r(s_1,a_1) \right] \\ &+ \E_{s_1 \sim d_0',a_1\sim \pi_\theta(\cdot | s_1)} \left[ \nabla_\theta \gamma V_{\mathcal{M'}}^{\pi_\theta}(s_2) \right] \\
=&  \E_{s_1 \sim d_0',a_1\sim \pi_\theta(\cdot | s_1)} \left[ \nabla_\theta  \ln\pi_\theta (a_1|s_1) Q_{\mathcal{M'}}^{\pi_\theta}(s_1,a_1)\right] +  \gamma \E_{s_1 \sim d_0',a_1\sim \pi_\theta(\cdot | s_1)} \left[ \nabla_\theta V_{\mathcal{M'}}^{\pi_\theta}(s_2) \right].
\end{align*}
Then we continue to recursively evaluate the expected value of $\nabla_\theta V_\mathcal{M'}^{\pi_\theta}(s_t)$ until we reach the final step:
\begin{align*}
    &\nabla_\theta J_{\mathcal{M'}}(\pi_\theta) \\
    =& \sum_{t=1}^{L-1} \gamma^{t-1} \E_{s_1\sim d_0', \pi_\theta}[\nabla_\theta \ln \pi_\theta(a_t|s_t) Q_{\mathcal{M'}}^{\pi_\theta}(s_t,a_t)] + \gamma^{L-1}\E_{s_1\sim d_0', \pi_\theta}[\nabla_\theta \ln \pi_\theta(a_L|s_L) Q_{\mathcal{M'}}^{\pi_\theta}(s_L,a_L)] \\
    =& \sum_{t=1}^{L-1} \gamma^{t-1} \E_{s_1\sim d_0', \pi_\theta}[\nabla_\theta \ln \pi_\theta(a_t|s_t) Q_{\mathcal{M'}}^{\pi_\theta}(s_t,a_t)] + \gamma^{L-1} \E_{s_1\sim d_0', \pi_\theta}[\nabla_\theta \ln \pi_\theta(a_L|s_L) r_L'(s_L,a_L)].
\end{align*}
Then since $r'_L(s,a) = Q_{\mathcal{M'}}^{\pi_\mathrm{ref}}(s,a)$ we get:
\begin{align*}
    &\nabla_\theta J_{\mathcal{M'}}(\pi_\theta) \\ =& \sum_{t=1}^{L-1} \gamma^{t-1} \E_{s_1\sim d_0', \pi_\theta}[\nabla_\theta \ln \pi_\theta(a_t|s_t) Q_{\mathcal{M'}}^{\pi_\theta}(s_t,a_t)] + \gamma^{L-1} \E_{s_1\sim d_0', \pi_\theta}[\nabla_\theta \ln \pi_\theta(a_L|s_L) Q_{\mathcal{M'}}^{\pi_\mathrm{ref}}(s_L,a_L)].
\end{align*}
Since we chose $\pi_\mathrm{ref} = \pi_\theta$ we get:
\begin{align*}
    \nabla_\theta J_{\mathcal{M'}}(\pi_\theta)  =& \sum_{t=1}^{L} \gamma^{t-1} \E_{s_1\sim d_0', \pi_\theta}[\nabla_\theta \ln \pi_\theta(a_t|s_t) Q_{\mathcal{M'}}^{\pi_\theta}(s_t,a_t)] \\
    =& \frac{1-\gamma^L}{1-\gamma}\E_{s\sim d_{\mathcal{M'}}^{\pi_\theta}, \pi_\theta}[\nabla_\theta \ln \pi_\theta(a_t|s_t) Q_{\mathcal{M'}}^{\pi_\theta}(s_t,a_t)].
\end{align*}
Combining this with Lemmas \ref{lemma:q} and \ref{lemma:occupancy} we get:
\begin{align*}
    \nabla_\theta J_{\mathcal{M'}}(\pi_\theta) =& \frac{1-\gamma^L}{1-\gamma}\E_{s\sim d_{\mathcal{M}}^{\pi_\theta}, \pi_\theta}[\nabla_\theta \ln \pi_\theta(a_t|s_t) Q_{\mathcal{M}}^{\pi_\theta}(s_t,a_t)].
\end{align*}
Moving to the LHS of the original equation, the policy gradient theorem \citep{sutton1998reinforcement} tells us: 
\begin{align*}
 \nabla_{\theta} J_{\mathcal{M}}(\pi_\theta) 
=& \frac{1}{1-\gamma} \E_{s \sim d_\mathcal{M}^{\pi_\theta}}\left [\E_{a \sim \pi_\theta(\cdot | s)} \left [ Q_\mathcal{M}^{\pi_\theta} (s,a) \nabla_\theta \ln\pi_\theta(a|s) \right ] \right ]
\end{align*}
Note that what we simplified $\nabla_{\theta}(J_{\mathcal{M'}}(\pi_\theta) - J_{\mathcal{M'}}(\pi'))$ too is simply a scalar multiple of $\nabla_{\theta} J_{\mathcal{M}}(\pi_\theta)$. Thus:
\begin{align*}
    \nabla_\theta J_\mathcal{M}(\pi_\theta) = \frac{1}{1-\gamma^L} \nabla_{\theta}(J_{\mathcal{M'}}(\pi_\theta) - J_{\mathcal{M'}}(\pi')),
\end{align*}
which concludes the proof.$\hfill\blacksquare$
\end{proof}

\subsection{Proof of Proposition \ref{prop:L_theorem_statement}}\label{sec:L_proof}

\textbf{Proposition 1} \begin{it} With probability $1-\delta$, the estimation error is bounded as follows:
    \begin{align*}
        \left|\widehat{J}_{\mathcal{M'}}(\pi_\theta, \mathcal{D}) - J_{\mathcal{M'}}(\pi_\theta)\right| \leq \frac{1}{1-\gamma} \sqrt{\frac{L\log(2/\delta)}{N}} + \gamma^{L-1} \nu.
    \end{align*}
 \end{it}
\begin{proof}
Recall that the estimator $\widehat{J}_{\mathcal{M'}}(\pi_\theta, \mathcal{D})$ consists of two terms:
\begin{align*}
   \widehat{J}_{\mathcal{M'}}(\pi_\theta, \mathcal{D}) =& \frac{1}{|D|} \sum_{i=1}^{|D|}R(\sigma^i) +  \gamma^{L-1}\frac{1}{|D|} \sum_{i=1}^{|D|} \widehat{Q}(s_L^i, a_L^i)\\
   =& \frac{1}{|D|} \sum_{i=1}^{|D|}\left(R(\sigma^i) + \gamma^{L-1}Q_{\mathcal{M}}^{\pi_{\mathrm{ref}}}(s_L^i, a_L^i)\right) \\&+ \gamma^{L-1}\frac{1}{|D|}\sum_{i=1}^{|D|}\xi_i.
\end{align*}
For the first term, since the trajectory segments $\sigma^i$ are sampled independently, we can use Hoeffding's bound as follows. Notice that the reward is bounded in $[0,1]$, the Q function is bounded in $[0,1/(1-\gamma)]$, and $|D| = N/2L$, then with probability $1-\delta$, we have:
\begin{align*}
    &\left|\frac{1}{|D|} \sum_{i=1}^{|D|}\left(R(\sigma^i) + \gamma^{L-1}Q_{\mathcal{M}}^{\pi_{\mathrm{ref}}}(s_L^i, a_L^i)\right) - J_{\mathcal{M'}}(\pi_\theta) \right| \\\leq& \left( \sum_{k=1}^{L-1}\gamma^{k-1} + \gamma^{L-1}\frac{1}{1-\gamma} \right) \sqrt{\frac{\log(2/\delta)}{2|D|}}\\
    =& \left (\frac{1-\gamma^{L-1}}{1-\gamma} + \frac{\gamma^{L-1}}{1-\gamma} \right )\sqrt{\frac{L\log(2/\delta)}{N}} \\
    =& \frac{1}{1-\gamma}\sqrt{\frac{L\log(2/\delta)}{N}}.
\end{align*}
The errors $\xi_i$ are adversarial and thus bounded by:
\begin{align*}
    \gamma^{L-1}\frac{1}{|D|}\sum_{i=1}^{|D|}\xi_i \leq \gamma^{L-1}\nu.
\end{align*}
So we obtain the final bound as:
\begin{align*}
    \left|\widehat{J}_{\mathcal{M'}}(\pi_\theta, \mathcal{D}) - J_{\mathcal{M'}}(\pi_\theta)\right| \leq \frac{1}{1-\gamma} \sqrt{\frac{L\log(2/\delta)}{N}} + \gamma^{L-1} \nu,
\end{align*}
which concludes the proof.$\hfill\blacksquare$
\end{proof}

\end{document}